\documentclass[10pt,journal,compsoc]{IEEEtran}
\ifCLASSOPTIONcompsoc
  \usepackage[nocompress]{cite}
\else
  \usepackage{cite}
\fi
\usepackage{amsmath}
\usepackage{amsthm}
\newtheorem{theorem}{Theorem}[section]
\newtheorem{proposition}{Proposition}[section]
\newtheorem{definition}{Definition}[section]
\usepackage{amssymb}
\usepackage{cases}
\usepackage{multirow}
\usepackage{graphicx}
\usepackage{epsfig}

\begin{document}
%
\title{Improving the Robustness and Generalization of Deep Neural Network with Confidence Threshold Reduction}
%
%
%
%

\author{Xiangyuan~Yang,
        Jie~Lin,
        Hanlin~Zhang,
        Xinyu~Yang,
        and~Peng~Zhao
\IEEEcompsocitemizethanks{\IEEEcompsocthanksitem Xiangyuan Yang, Jie Lin, Xinyu Yang and Peng Zhao are with School of Computer Science and Technology, Xi'an Jiaotong University, Xi'an, China. \protect
E-mail: ouyang\_xy@stu.xjtu.edu.cn,
\{jielin, yxyphd, p.zhao\}@mail.xjtu.edu.cn
\IEEEcompsocthanksitem Hanlin Zhang is with Qingdao University, Qingdao, China. \protect E-mail: hanlin@qdu.edu.cn
\IEEEcompsocthanksitem Corresponding author: Jie Lin.}
}

%
%

\markboth{Journal of \LaTeX\ Class Files,~Vol.~14, No.~8, August~2015}%
{Yang \MakeLowercase{\textit{et al.}}: Improving the Robustness and Generalization of Deep Neural Network with Confidence Threshold Reduction}
%



\IEEEtitleabstractindextext{%
\begin{abstract}
Deep neural networks are easily attacked by imperceptible perturbation. Presently, adversarial training (AT) is the most effective method to enhance the robustness of the model against adversarial examples. However, because adversarial training solved a min-max value problem, in comparison with natural training, the robustness and generalization are contradictory, i.e., the robustness improvement of the model will decrease the generalization of the model. To address this issue, in this paper, a new concept, namely confidence threshold (CT), is introduced and the reducing of the confidence threshold, known as confidence threshold reduction (CTR), is proven to improve both the generalization and robustness of the model. Specifically, to reduce the CT for natural training (i.e., for natural training with CTR), we propose a mask-guided divergence loss function (MDL) consisting of a cross-entropy loss term and an orthogonal term. The empirical and theoretical analysis demonstrates that the MDL loss improves the robustness and generalization of the model simultaneously for natural training. However, the model robustness improvement of natural training with CTR is not comparable to that of adversarial training. Therefore, for adversarial training, we propose a standard deviation loss function (STD), which minimizes the difference in the probabilities of the wrong categories, to reduce the CT by being integrated into the loss function of adversarial training. The empirical and theoretical analysis demonstrates that the STD based loss function can further improve the robustness of the adversarially trained model on basis of guaranteeing the changeless or slight improvement of the natural accuracy.
\end{abstract}

}

\maketitle

\IEEEdisplaynontitleabstractindextext

%
\IEEEpeerreviewmaketitle

\IEEEraisesectionheading{\section{Introduction}\label{sec:introduction}}

%
%
%
%
\IEEEPARstart{D}{eep} learning has penetrated various daily applications, and how to improve the performance of the deep neural network (DNN) has received wide attentions. Due to the existence of overfitting, the DNN does not realize its full potential. In addition, the imperceptible perturbations~\cite{FGSM,PGD,CW,BIM,MIFGSM,AutoAttack,Square,AoA,UAP} lead the misclassification of the DNN, resulting in a significant reduction on classification accuracy. Therefore, increasing both the generalization and robustness of the DNN still needs to be further investigated.

To improve the generalization of the DNN avoiding overfitting, dropout~\cite{Dropout} was proposed to prevent co-adaptation of feature detectors by omitting a part of neurons with a fixed probability. Subsequently, many variants of dropout gradually have been developed. For example, dropconnect~\cite{Dropconnect} improves the generalization of neural networks by omitting links between neurons of adjacent layers, and dropblock~\cite{Dropblock} omits regions in the convolutional layer of the DNN. Multi-sample dropout~\cite{Multi-Sample-Dropout} was proposed to accelerate training and obtain better generalization, which dropouts the feature layer in the DNN multiple times and calculates the average loss in every feedforward and backward training step. Disout~\cite{Beyond-Dropout} was proposed to replace with dropout, which introduces feature map distortion to reduce the Rademacher complexity of the DNN for improving generalization. In addition, Srivastava et al.~\cite{Dropout-Prevent-Overfitting} and Konda et al.~\cite{Dropout-Data-Augmentation} explained that dropout can improve the generalization of the DNN through ensemble model and data augmentation. However, these regularization methods can not guarantee the robustness of the DNN.

To enhance the robustness of the DNN in resisting adversarial attacks, adversarial training~\cite{PGD} has been developed as the most effective defense method. Adversarial training methods can be divided into two categories based on the purpose: the only robustness improvement of the DNN and the trade-off between the robust accuracy and natural accuracy. The former includes Madry adversarial training (Madry-AT)~\cite{PGD}, Free adversarial training (Free-AT)~\cite{Free-AT} and Fast adversarial training (Fast-AT)~\cite{Fast-AT}. Madry-AT first summarized the adversarial training methods as the min-max value formulation. Free-AT~\cite{Free-AT} was proposed to eliminate the overhead cost of generating adversarial examples by recycling the gradient information computed. Fast-AT~\cite{Fast-AT} is a much weaker and cheaper adversary without more costly to natural training. The later includes the trade-off between robustness and accuracy (TRADES)~\cite{TRADES} and its variants, i.e., adversarial training with transferable adversarial examples (ATTA)~\cite{ATTA} and friendly adversarial training (FAT)~\cite{FAT}. TRADES~\cite{TRADES} was a new defense algorithm to identify a trade-off between robustness and accuracy. ATTA~\cite{ATTA} improves the training efficiency by accumulating adversarial perturbations through epochs. FAT~\cite{FAT} was proposed to minimize the loss by using the least adversarial data that are confidently misclassified. However, in comparison with natural training, all adversarial training methods will greatly reduce the natural accuracy because the robustness and generalization are contradictory in the min-max value problem.

According to the existing efforts mentioned above, we conclude two issues: i) the regularization methods cannot enhance the robustness of the model, and ii) the adversarial training methods will decrease the generalization of the model. To address these two issues, we propose the mask-guided divergence loss function (MDL) for natural training and the standard deviation loss function (STD) for adversarial training, respectively, to improve both the robustness and generalization of DNN simultaneously.

The contributions of our paper can be summarized as follows:

First, the new concept of confidence threshold (CT) is introduced in our paper. Through reducing of the CT ( known as confidence threshold reduction, CTR), the issues that the regularization methods cannot improve the robustness of the model, as well as the adversarial training methods may greatly reduce the generalization of the model, can be mitigated, and thus both the robustness and generalization of DNN is proven to be enhanced simultaneously.

Second, to reduce the CT for natural training, we propose the mask-guided divergence loss function (MDL), which consists of the cross-entropy loss term and the orthogonal term. The theoretical analysis demonstrates that the MDL loss can reduce the CT. Additionally, the experimental analysis demonstrates that the MDL loss can enhance both the generalization and robustness of the model simultaneously on two deep neural networks, four datasets, six adversarial attack algorithms and nine regularization methods.

Third, because the model robustness improvement of natural training with CTR is not comparable to that of adversarial training, we apply CTR into adversarial training to further improve the robustness of the adversarially trained model while keeping the generalization unchanged or improved. To reduce the CT for adversarial training, the standard deviation loss function (STD) is proposed to be integrated into the loss function of adversarial training to minimize the difference in the probabilities of the wrong categories. The theoretical analysis demonstrates that, with STD loss, the confidence threshold (CT) can be reduced in adversarial training, and the generated adversarial examples with the variants of the STD loss are more effective than those without the STD loss. Additionally, the experimental analysis demonstrates that, the STD loss can assist in further improving the robustness of the adversarially trained model on the basis of guaranteeing the changeless or slight improvement of generalization of the model on a deep neural network, a dataset, four adversarial training methods and eleven adversarial attack algorithms.

Lastly, three STD loss based adversarial attack algorithms, namely the STD loss based fast gradient sign method (S-FGSM), project gradient descent (S-PGD) and automatic projected gradient descent (S-APGD), respectively, are proposed to replace the cross-entropy (CE) loss in fast gradient sign method (FGSM), projected gradient descent (PGD) and Automatic PGD (APGD) with the STD loss. Through the theoretical and experimental analysis, the results show that the generated adversarial examples with the STD loss are more effective than the CE loss, and the attack success rates of S-FGSM, S-PGD and S-APGD are higher than that of FGSM, PGD and APGD, respectively, on both the naturally trained model and the adversarially trained model.

The rest of this paper is organized as follows. In Section~\ref{sec:preliminaries}, we briefly introduce multi-sample dropout~\cite{Multi-Sample-Dropout} and several adversarial training methods, which are used in our method. Section~\ref{sec:motivations} explains two motivations of our paper. According to the two motivations of Section~\ref{sec:motivations}, Section~\ref{sec:methodology} proposes the mask guided divergence loss (MDL) for natural training, the standard deviation loss (STD) for adversarial training and three the STD loss based adversarial attacks, respectively. Section~\ref{sec:experiments} evaluates the proposed methods on several benchmark datasets. Section~\ref{sec:related-work} conducts the related work and Section~\ref{sec:conclusion} concludes the whole paper.

\section{Preliminaries}
\label{sec:preliminaries}
In this section, the multi-sample dropout~\cite{Multi-Sample-Dropout} and several adversarial training (AT) methods are briefly introduced, which are helpful to understand our methods in Section~\ref{sec:methodology} and are regarded as the baselines in Section~\ref{sec:experiments}. In addition, A proposition~\cite{QIFGSM} is introduced to analyze our methods in Section~\ref{sec:methodology}.

\subsection{Multi-Sample Dropout}
\label{sec:multi-sample-dropout}
In multi-sample dropout~\cite{Multi-Sample-Dropout}, the last feature layer in the DNN dropouts multiple times in each feedforward, and the average loss is calculated by all outputs. Then a feedback is used to update the parameters of the DNN, which are equivalent to training multiple sub-DNNs in each training step of the DNN where the sub-DNN is the remaining part of the DNN after dropout. Generally, in the DNN with dropout, only one sub-DNN is trained in each training step. Therefore, in the same training time, multi-sample dropout can make the DNN converge faster and achieve better and higher generalization performance than dropout.

\subsection{Adversarial Training Methods}
\label{sec:adversarial-training-methods}
\textbf{Madry-AT}~\cite{PGD} is the first projected gradient descent (PGD) attack based adversarial training method, which generates adversarial examples by PGD-7 as the training dataset where PGD-7 denotes running the PGD attack with 7 steps. Madry et al.~\cite{PGD} first summarize the adversarial training methods as the min-max value formulation, i.e.,
\begin{align}\label{eq:madry-at}
\underset{f\in \mathcal{H}}{\min}\mathbb{E} _{\left( x,y_{x} \right) \sim \mathcal{D}}\left[ \underset{\delta \in \mathcal{B}}{\max}L_{CE}\left( f\left( x+\delta \right) ,y_{x} \right) \right]
\end{align}
where $\mathcal{H}$ is the hypothesis space, $\mathcal{D}$ is the distribution of the training dataset, $L_{CE}$ is the cross-entropy loss function, and $\mathcal{B}$ is the allowed perturbation space that is usually selected as an L-$p$ norm ball around $x$. Specifically, the min-max value formulation can be decomposed into training objective $L_{CE}(f(x^{\ast}),y_{x})$ and adversarial objective $L_{CE}(f(x'),y_{x})$ where $x^{\ast}$ is the generated adversarial example and $x'$ is the immature adversarial example. The adversarial examples are generated by maximizing the adversarial objective. Then, the model is trained by the generated adversarial examples to minimize the training objective.

\textbf{Free-AT}~\cite{Free-AT} and \textbf{Fast-AT}~\cite{Fast-AT} are the accelerated versions of adversarial training. To accelerate convergence, Free-AT~\cite{Free-AT} improves Madry-AT~\cite{PGD} by reusing feedback gradients. Fast-AT~\cite{Fast-AT} has proved that the adversarial examples generated by FGSM, known as a single step attack, can be used effectively for adversarial training with greatly reducing the training time.

\textbf{TRADES}~\cite{TRADES} trained a DNN model with both natural and adversarial examples to trade-off natural and robust errors. The min-max value formulation is changed as follows:
\begin{align}\label{eq:trades}
\underset{f\in \mathcal{H}}{\min}\mathbb{E} _{\left( x,y_{x} \right) \sim \mathcal{D}}\left[ \begin{array}{c}
L_{CE}\left( f\left( x \right) ,y_{x} \right)\\
+\beta \cdot \underset{\delta \in \mathcal{B}}{\max}L_{KL}\left( f\left( x+\delta \right) ,f\left( x \right) \right)\\
\end{array} \right]
\end{align}
where $L_{KL}$ is the Kullback-Leibler divergence loss, $\beta$ is a regularization parameter that controls the trade-off between the natural and robust accuracies. When $\beta$ increases, the natural accuracy decreases and the robust accuracy increases, and vice versa. Specifically, the adversarial objective is $L_{KL}\left( f\left( x' \right) ,f\left( x \right) \right)$ and the training objective is $L_{CE}\left( f\left( x \right) ,y_{x} \right) +\beta \cdot L_{KL}\left( f\left( x^{\ast} \right) ,f\left( x \right) \right)$.

\subsection{Proposition}\label{sec:proposition}
The proposition~\cite{QIFGSM} found that the adversarial examples generated with the gradient attack methods prefer to be classified as the wrong categories with higher probability. Meanwhile, the higher the probability of the wrong category, the greater the weight of the adversarial examples classified as the wrong category.

\section{Motivations}
\label{sec:motivations}
\subsection{Motivation I: Enhancing the Diversity of the Emsemble Sub-DNN}
\label{sec:motivation-I}
Deep neural network (DNN) with dropout can be regarded as an ensemble model consisting of lots of sub-DNNs~\cite{Dropout-Prevent-Overfitting} (i.e., an ensemble sub-DNN where the sub-DNN is the remaining part of the DNN after dropout), and through increasing the diversity of the ensemble sub-DNN, the generalization and robustness of the DNN can be effectively improved.

\subsection{Motivation II: Introducing of Confidence Threshold Reduction}
\label{sec:motivation-II}
\begin{definition}[Confidence threshold $\hat{p}$]
\label{def:CT}
When the predicted probability of the correct category is greater than $\hat{p}$, the input almost be classified correctly. However, when the predicted probability of the correct category is less than or equal to $\hat{p}$, the input may be classified correctly or incorrectly. Therefore, the confidence threshold (CT) of the classifier is $\hat{p}$.
\end{definition}

For example, the confidence threshold of the most known classifiers at present is $\hat{p}=0.5$. When $p_{y_x}>0.5$, the input almost be classified correctly where $p_{y_x}$ is the predicted probability of the correct category $y_x$. When $p_{y_x}\leq 0.5$, the input may be classified correctly or incorrectly.

\begin{proposition}
\label{prop:generalization}
The confidence threshold reduction (i.e., reducing the CT) can improve the generalization of the model.
\end{proposition}

\begin{proof}
On the basis of ensuring the correct classification of the inputs whose original probability of the correct category is greater than 0.5, reducing the confidence threshold $\hat{p}$ (i.e., $\hat{p}<0.5$) can make more inputs whose predicted probability is in the range $(\hat{p}, 0.5]$ correct. Therefore, the generalization of the model is improved.
\end{proof}

\begin{proposition}
\label{prop:robustness}
The confidence threshold reduction can enhance the robustness of the model.
\end{proposition}

\begin{proof}
Similarly, on the basis of ensuring the correct classification of the inputs whose original probability of the correct category is greater than 0.5, in comparison with $\hat{p}=0.5$, after reducing the confidence $\hat{p}$, a stronger attack is required to reduce the predicted probability of the correct category until $p_{y_x}\leq \hat{p} < 0.5$. Therefore, the robustness of the model is improved.
\end{proof}

Overall, according to Propositions~\ref{prop:generalization} and \ref{prop:robustness}, the confidence threshold reduction jointed into the original training can improve the generalization and robustness of the model simultaneously.

\section{Methodology}
\label{sec:methodology}
In this section, the confidence threshold reduction (CTR) applied in natural training and adversarial training is described in details, respectively, in which the mask-guided divergence loss (MDL) is proposed for natural training and the standard deviation loss (STD) is proposed for adversarial training. Finally, three STD loss based attack methods are introduced.

\subsection{Natural Training with Confidence Threshold Reduction}
\label{sec:natural-training-with-CTR}
We first propose mask-guided divergence loss function (MDL), which is inspired by motivation I (i.e., enhancing the diversity of the ensemble sub-DNN). Then, the theoretical analysis demonstrates that the MDL loss can satisfy the requirement of motivation II (i.e., reducing the CT).

\subsubsection{Mask-Guided Divergence Loss (MDL)}
\label{sec:MDL}
In this section, the proposed mask-guided divergence loss function is described in detail. The basic idea is that because the DNN with dropout can be regarded as an ensemble sub-DNN, the generalization and robustness of the DNN can be improved by increasing the diversity of the ensemble sub-DNN. However, to increase the diversity of the ensemble sub-DNN, two problems should be solved: i) in each training step of the DNN with dropout, only one sub-DNN updates the parameters, which can not actively enhance the diversity of the ensemble sub-DNN, and ii) how to define the diversity among sub-DNNs. 

To actively enhance the diversity of the ensemble sub-DNN (i.e., problem i), multi-sample dropout~\cite{Multi-Sample-Dropout} can be used to dropout the last feature layers multiple times in each training step, and then the outputs of multiple sub-DNNs are jointly calculated to increase the diversity among sub-DNNs.

To define the diversity among sub-DNNs (i.e., problem ii), the following statement can be devoted. The output of the last convolutional layer is input to dropout two times. After going through all fully connected layers and softmax, two classification results can be obtained, i.e., the outputs of the two sub-DNNs. Then, the mark is conclulated by the correct predicted probabilities of the two sub-DNNs, and the diversity is conclulated with considering the wrong category predicted vectors of the two sub-DNNs, where the mask and diversity constitute the orthogonal term in the proposed mask-guided divergence loss (MDL). 

Assuming that the wrong category predicted vectors for the input $(x,y_x)$ that are output by the first and second sub-DNNs are denoted as $\boldsymbol{\bar{P}}_x^1=\boldsymbol{P}_x^1\backslash y_x$ and $\boldsymbol{\bar{P}}_x^2=\boldsymbol{P}_x^2\backslash y_x$, respectively, where $\boldsymbol{P}_x^1$ and $\boldsymbol{P}_x^2$ are the predicted probability vectors of the first sub-DNN and the second sub-DNN, respectively, $\backslash y_x$ denotes removing the correct probability from the predicted probability vector. Here, cosine similarity can be involved to define the diversity among sub-DNNs:
the cosine similarity based diversity $D_1(\boldsymbol{P}_x^1, \boldsymbol{P}_x^2)$ can be represented as
\begin{align}\label{eq:CS-based-diversity}
D_1(\boldsymbol{P}_x^1, \boldsymbol{P}_x^2) = \frac{\boldsymbol{\bar{P}}_x^1\cdot \boldsymbol{\bar{P}}_x^2}{\Vert \boldsymbol{\bar{P}}_x^1\Vert _2\Vert \boldsymbol{\bar{P}}_x^2\Vert _2}
\end{align}
where $\Vert\cdot\Vert_2$ is \textit{l2} norm. The smaller $D_1(\boldsymbol{P}_x^1, \boldsymbol{P}_x^2)$, the greater the diversity.

Generally, if there are $K$ sub-DNNs in one training step, the calculation formula of the orthogonal term with the cosine similarity based diversity in the mask-guided divergence loss can be represented as
\begin{align}\label{eq:orthogonal-term}
O\left( \boldsymbol{P}_{x}^{1},\boldsymbol{P}_{x}^{2},\cdots,\boldsymbol{P}_{x}^{K}\right|\eta) =\xi \left( \boldsymbol{P}_x \right|\eta) \sum_{i=1}^K{\sum_{j=i+1}^K{D_1\left( \boldsymbol{P}_{x}^{i},\boldsymbol{P}_{x}^{j} \right)}}
\end{align}
where $\boldsymbol{P}_x^i$ denotes the predicted probability vector of the $i^{th}$ sub-DNN, the mask $\xi(\boldsymbol{P}_x|\eta)$ ensures that a certain percentage of the inputs (denoted as $\eta\%$) with high correct predicted probabilities in each batch size of training inputs participate in the calculation of the orthogonal term, which can be represented as
\begin{align}
\xi \left( \boldsymbol{P}_x \right|\eta)=\left\{ \begin{array}{c}
	1, -\log \left( p_{x}^{y_x} \right) \le T\\
	0, -\log \left( p_{x}^{y_x} \right) > T\\
\end{array} \right.
\end{align}
where $\boldsymbol{P}_x$ is the average predicted probability vector of the $K$ sub-DNNs in the training process, $p^{y_x}_x$ is a element in $\boldsymbol{P}_x$ and represents the predicted probability of the category $y_x$, $T$ is $\eta$ percentile of the set $\left\{-\log(p_{x_1}^{y_{x_1}}),-\log(p_{x_2}^{y_{x_2}}),\cdots,-\log(p_{x_B}^{y_{x_B}})\right\}$ where $B$ is the batch size in training, and $x_i$ is an input in the batch size of the training inputs. Note that, the effect of the mask is to avoid overfitting because the existence of inputs that are difficult to learn in the training set can make the learned diversity have a more complex expression.

Finally, for the input $(x,y_x)$, the mask-guided divergence loss consists of the average cross-entropy loss of $K$ sub-DNNs (i.e., the cross-entropy loss term) and the orthogonal term, which can be represented as
\begin{align}\label{eq:L-MDL}
\begin{array}{c}
L_{MDL}\left(f\left(x\right), y_{x}|K,\rho,\eta \right)=\underset{cross-entropy\,\,loss\,\,term}{\underbrace{\frac{\sum_{i\leqslant K}{L_{CE}^{i}(f(x),y_x)}}{K}}}+\\
\rho \cdot \underset{orthogonal\,\,term}{\underbrace{\frac{O\left( \boldsymbol{P}_{x}^{1},\boldsymbol{P}_{x}^{2},\cdots ,\boldsymbol{P}_{x}^{K} |\eta \right)}{\frac{K\cdot \left( K-1 \right)}{2}}}}
\end{array}
\end{align}
where $L_{CE}^i$ is the cross-entropy loss of the $i$-th sub-DNN, and $\rho$ is a weight coefficient. $\frac{K\cdot (K-1)}{2}$ in the orthogonal term of Eq.~\ref{eq:L-MDL} ensures the value of the orthogonal term in the range $[0,1]$.

\subsubsection{Theoretical Analysis}
\label{sec:theoretical-analysis-MDL}
In this section, we will prove why the MDL loss can improve the generalization and robustness of the DNN. For an input image $x$, the correct category $y_x$ is assumed as the $C$-th category where $C$ is the number of categories, and the wrong category predicted vector of the $i$-th sub-DNN is $[p^{i1}_x, p^{i2}_x, \cdots, p^{i(C-1)}_x]$ where $p^{ic}_x$ denotes the predicted probability of the category $c$ of the $i$-th sub-DNN for the input $x$.

\begin{theorem}
\label{theorem:cosine-maximum}
When the MDL loss function converges $L_{MDL}\rightarrow 0$, the diversity of all sub-DNNs (assume $N$ sub-DNNs) is maximized. Therefore, for the input image $s$, each axis is paralleled by $\frac{N}{C-1}$ sub-DNNs' wrong categories predicted vectors where $C$ is the number of the category.
\end{theorem}

\begin{proof}
Theorem~\ref{theorem:cosine-maximum} can be simplified as: the scheme of minimizing the sum of cosine values between $N$ number of $(C-1)$-dimensional vectors is that each axis is paralleled by $\frac{N}{C-1}$ vectors. 

First, there are only $0$ and $\frac{\pi}{2}$ angles between the vectors, because the existence of the acute angle will increase the sum of cosine values. For example, when $a_1$ number of $0$ and $a_2$ number of $\frac{\pi}{2}$ angles are adjusted to $a_1+a_2$ acute angles $\theta^{a_1}_1, \theta^{a_1}_2, \cdots, \theta^{a_1}_{a_1+a_2}$, the minimum sum of cosine values can be represented as
\begin{align}\label{eq:acute-angle-increase-the-sum}
\left\{ \begin{array}{c}
\underset{\theta _{1}^{a_1},\theta _{2}^{a_1},\cdots}{\min}\cos \theta _{1}^{a_1}+\cos \theta _{2}^{a_1}+\cdots +\cos \theta _{a_1+a_2}^{a_1}>a_1\\
s.t.\left\{ \begin{array}{c}
\theta _{1}^{a_1}+\theta _{2}^{a_1}+\cdots +\theta _{a_1+a_2}^{a_1}=\frac{a_2}{2}\pi\\
\forall 1\leqslant i\leqslant a_1+a_2,0<\theta _{i}^{a_1}<\frac{\pi}{2}\\
\theta _{1}^{a_1}\geqslant \theta _{2}^{a_1}\geqslant \cdots \geqslant \theta _{a_1+a_2}^{a_1}\\
\end{array} \right.\\
\end{array} \right.,
\end{align}
which is proved by mathematical induction as follow: when $a_1=1$, if $a_2$ is an odd number, the sum of cosine values can be proven to be larger than 1, i.e., 
\begin{equation}
\begin{aligned}\label{eq:odd-number-case}
&\cos \theta _{1}^{1}+\cos \theta _{2}^{1}+\cdots +\cos \theta _{1+a_2}^{1}>\\
&\pm \cos \left( \theta _{1}^{1}+\theta _{2}^{1}+\cdots +\theta _{a_2}^{1} \right) +\cos \theta _{1+a_2}^{1}
\\
&=\pm \cos \left( \frac{a_2}{2}\pi -\theta _{1+a_2}^{1} \right) +\cos \theta _{1+a_2}^{1}\\
&=\sin \theta _{1+a_2}^{1}+\cos \theta _{1+a_2}^{1}>1
\end{aligned}
\end{equation}
if $a_2$ is an even number, 
\begin{equation}
\begin{aligned}\label{eq:even-number-case}
&\cos \theta _{1}^{1}+\cos \theta _{2}^{1}+\cdots +\cos \theta _{1+a_2}^{1}>\\
&\pm \sin \left( \theta _{1}^{1}+\theta _{2}^{1}+\cdots +\theta _{a_2}^{1} \right) +\cos \theta _{1+a_2}^{1}
\\
&=\pm \sin \left( \frac{a_2}{2}\pi -\theta _{1+a_2}^{1} \right) +\cos \theta _{1+a_2}^{1}\\
&=\sin \theta _{1+a_2}^{1}+\cos \theta _{1+a_2}^{1}>1
\end{aligned}
\end{equation}
where Eq.~\ref{eq:odd-number-case} and Eq.~\ref{eq:even-number-case} can be obtained by the undetermined coefficient method. Therefore, when $a_1=1$, Eq.~\ref{eq:acute-angle-increase-the-sum} is satisfied. Assuming $a_1=k$, Eq.~\ref{eq:acute-angle-increase-the-sum} is also satisfied, i.e.,
\begin{align}
\cos \theta _{1}^{k}+\cos \theta _{2}^{k}+\cdots +\cos \theta _{k+a_2}^{k}>k
\end{align}
When $a_1=k+1$, the sum of cosine values can be proven to be larger than $k+1$, i.e.,
\begin{equation}
\begin{aligned}
&\cos \theta _{1}^{k+1}+\cos \theta _{2}^{k+1}+\cdots +\cos \theta _{k+1+a_2}^{k+1}
\\
&=\cos \theta _{1}^{k}+\cos \theta _{2}^{k}+\cdots +\cos \theta _{k+a_2}^{k}+1+
\\
&\left( \cos \theta _{1}^{k+1}-\cos \theta _{1}^{k} \right)+\left( \cos \theta _{2}^{k+1}-\cos \theta _{2}^{k} \right) +\cdots +
\\
&\left( \cos \theta _{k+a_2}^{k+1}-\cos \theta _{k+a_2}^{k} \right) +\left( \cos \theta _{k+1+a_2}^{k+1}-1 \right)>k+1+
\\
&\theta _{k+1+a_2}^{k+1}\cdot \left( \sin \left( \underset{i\leqslant k+a_2}{\min}\theta _{i}^{k+1} \right) -\sin \left( \theta _{k+1+a_2}^{k+1} \right) \right) >k+1
\end{aligned}
\end{equation}
Therefore, Eq.~\ref{eq:acute-angle-increase-the-sum} is successful. 

Second, according to the mean value theorem, when the number of vectors paralleled to each axis is equal to each other, i.e., equal to $\frac{N}{C-1}$, the sum of cosine values will be the smallest. Hence, in summary, the Theorem~\ref{theorem:cosine-maximum} is proved.
\end{proof}

According to Theorem~\ref{theorem:cosine-maximum}, after the MDL loss converges, for the input image $x$,
\begin{align}
\left\{ \begin{array}{c}
0\leq p_{x}^{ij}\leq 1\\
0\leq p_{x}^{1j}+p_{x}^{2j}+\cdots +p_{x}^{Nj}\leq \frac{N}{C-1}\\
\end{array} \right..
\end{align}
Therefore, the probability of $j$-th wrong category predicted by the DNN trained with MDL is 
\begin{align}
p_{x}^{j}=\frac{p_{x}^{1j}+p_{x}^{2j}+\cdots +p_{x}^{Nj}}{N}\leqslant \frac{\frac{N}{C-1}}{N}=\frac{1}{C-1}
\end{align}
where $1\leqslant j \leqslant C-1$. When $C$ is greater than 2, the confidence threshold is reduced to $\frac{1}{C-1}$ which is less than 0.5. Therefore, the MDL loss can reduce the confidence threshold when the MDL loss converges to 0. 

For example, for a 10-category task (e.g., MNIST and CIFAR10), $p^j_x \leqslant \frac{1}{9}$. Ideally, if $p^C_x$ is greater than $\frac{1}{9}$, the input image $x$ can confidently be classified correctly. However, for the DNN trained without MDL, the input image $x$ can confidently be classified correctly if and only if $p^C_x$ is greater than $\frac{1}{2}$. Therefore, according to Proposition~\ref{prop:generalization}, MDL can increase the generalization of the DNN. In the case of being attacked, the greater attack strength is required to make $p^C_x$ less than $\frac{1}{9}$ by comparing with $\frac{1}{2}$. Hence, according to Proposition~\ref{prop:robustness}, the robustness of the DNN is also enhanced.

\subsection{Adversarial Training with Confidence Threshold Reduction}
\label{sec:adversarial-training-with-CTR}
Due to the underfitting of adversarial training on the natural training inputs, the known adversarial training methods~\cite{PGD,TRADES,Fast-AT,ATTA,FAT,Free-AT} do not use dropout~\cite{Dropout} and its variants~\cite{Dropconnect,Dropblock,Multi-Sample-Dropout,Auto-Dropout,Contextual-Dropout,Jumpout,Structured-Dropout,Meta-Dropout,Beyond-Dropout,Weighted-Channel-Dropout,Guided-Dropout,Message-Dropout,Adversarial-Dropout} to improve the generalization of the adversarially trained model. Because the MDL loss needs to be used together with dropout or dropout variants to reduce the CT, the MDL loss cannot be applied to the existing adversarial training methods to further improve the robustness of the model. Therefore, to reduce the CT of the existing adversarial training methods, we propose the standard deviation loss function.

\subsubsection{Standard Deviation Loss (STD)}
\label{sec:STD}
In fact, the confidence threshold reduction can be implemented not only by combining with dropout or dropout variants, but also by reducing the difference in the probability of the wrong categories. In this paper, the standard deviation is used to measure the difference in the probability of the wrong categories, i.e., $S(f(x)\backslash y_x)$ where $S(\cdot)$ denotes calculating the standard deviation of the vector, and $f(x)\backslash y_x$ denotes the wrong categories probability vector. Therefore, to reduce the confidence threshold in adversarial training, the standard deviation loss (STD) is proposed:
\begin{align}\label{eq:STD}
L_{STD}(f(x),y_x)=S(f(x)\backslash y_x).
\end{align}
Then, the STD loss based cross-entropy loss (SCE) is achieved by applying the STD loss into the cross-entropy loss:
\begin{align}\label{eq:SCE}
L_{SCE}(f(x),y_x|\gamma)=e^{\gamma\cdot L_{STD}(f(x),y_x)}\cdot L_{CE}(f(x),y_x).
\end{align}
The STD loss based Kullback-Leibler divergence loss (SKL) is achieved by applying the STD loss into the Kullback-Leibler divergence loss:
\begin{equation}
\begin{aligned}\label{eq:SKL}
&L_{SKL}(f(x^{\ast}),f(x),y_x|\gamma)=\\
&e^{\gamma\cdot L_{STD}(f(x),y_x)}\cdot L_{KL}(f(x^{\ast}),f(x)).
\end{aligned}
\end{equation}
In Eq.~\ref{eq:SCE} and \ref{eq:SKL}, the term $e^{\gamma\cdot L_{STD}(f(x),y_x)}$ is used to reduce the confidence threshold where the parameter $\gamma$ controls the weight of the term, and the larger the parameter $\gamma$, the greater the weight. Due to $e^{\gamma\cdot L_{STD}(f(x),y_x)}\geq 1$, the convergence of $L_{STD}(f(x),y_x)$ will not influence the convergence of $L_{CE}$ in $L_{SCE}$ or $L_{KL}$ in $L_{SKL}$ eventually.

Therefore, the optimization formulation obtained by applying the STD loss into Madry-AT~\cite{PGD}, Free-AT~\cite{Free-AT} and Fast-AT~\cite{Fast-AT}, which also calls Madry-AT with CTR, Free-AT with CTR and Fast-AT with CTR, respectively, can be represented as 
\begin{align}\label{eq:madry-at-with-CTR}
\underset{f\in \mathcal{H}}{\min}\mathbb{E} _{\left( x,y_{x} \right) \sim \mathcal{D}}\left[ \underset{\delta \in \mathcal{B}}{\max}L_{SCE}\left( f\left( x+\delta \right) ,y_{x}|\gamma \right) \right],
\end{align}
in which the training objective is $L_{SCE}(f(x^{\ast}),y_x|\gamma)$ and the adversarial objective is $L_{SCE}(f(x'),y_x|\gamma)$.

Similarly, the optimization formulation by applying the STD loss into TRADES~\cite{TRADES}, which also calls TRADES with CTR, can be represented as
\begin{align}\label{eq:trades-with-CTR}
\underset{f\in \mathcal{H}}{\min}\mathbb{E} _{\left( x,y_{x} \right) \sim \mathcal{D}}\left[ \begin{array}{c}
L_{SCE}\left( f\left( x \right) ,y_{x}|\gamma \right)+\\
\beta \cdot \underset{\delta \in \mathcal{B}}{\max}L_{SKL}\left( f\left( x+\delta \right) ,f\left( x \right),y_x|\gamma \right)\\
\end{array} \right],
\end{align}
in which the training objective is $L_{SCE}(f(x),y_x)+\beta\cdot L_{SKL}(f(x^{\ast}),f(x),y_x|\gamma)$ and the adversarial objective is $L_{SKL}(f(x'),f(x),y_x|\gamma)$.

Note that because the calculation of the STD loss does not use the predicted probability of the correct category, at the initial training stage, the convergence of the STD function in the SCE loss (or the SKL loss) may lead to the non-convergence of the CE function (or the KL function) in the SCE loss (or the SKL loss). To this end, the gradual warmup strategy~\cite{Warmup} can be used at the initial training stage of the DNN with the SCE or SKL loss to solve the challenge of the initial training difficulties. In addition, the Cyclic learning rate strategy~\cite{Cyclic} that has the property of the gradual warmup~\cite{Warmup} can be used to train the DNN with the SCE or SKL loss as well.

\subsubsection{Theoretical Analysis}
\label{sec:adversarial-training-with-CTR-theoretical-analysis}
In this section, we first demonstrate that the convergence of the STD loss can effectively reduce the confidence threshold (i.e., Proposition~\ref{prop:CTR-with-STD}). Then, we demonstrate that the generated adversarial examples with the variants of the STD loss for training are more effective than those without the STD loss (i.e., Proposition~\ref{prop:MadryAT-with-CTR-better-than-without} for Madry-AT and Proposition~\ref{prop:TRADES-with-CTR-better-than-without} for TRADES), which is the another resason why the STD loss can enhance the robustness of the adversarially trained model. Note that the SCE and SKL losses are the variants of the STD loss.

\begin{proposition}\label{prop:CTR-with-STD}
In the optimization Eq.~\ref{eq:madry-at-with-CTR} and \ref{eq:trades-with-CTR} of adversarial training, when the STD loss function converges $L_{STD}\rightarrow 0$, the confidence threshold will reduce to $\frac{1}{C-1}$ from 0.5.
\end{proposition}

\begin{proof}
For an input image $x$, the correct category $y_x$ is assumed as the $C$-th category where $C$ is the number of the categories (The assumption is used in the proof of Proposition~\ref{prop:MadryAT-with-CTR-better-than-without} and \ref{prop:TRADES-with-CTR-better-than-without}). When the STD loss converges to 0, the relationship among the predicted probabilities of the wrong categories is that
\begin{align}
\left\{ \begin{array}{c}
p_{x}^{1}\approx p_{x}^{2}\approx \cdots \approx p_{x}^{C-1}\\
p_{x}^{1}+p_{x}^{2}+\cdots +p_{x}^{C-1}\leq 1\\
\end{array} \right.
\end{align}
where $p_{x}^{j}$ is the probability of $j$-th wrong category predicted by the DNN trained with the STD loss. Therefore, $p_{x}^{j}$ satisfies that
\begin{align}
p_{x}^{j}\leq \frac{1}{C-1}
\end{align}
where $1\leq j\leq C-1$. When $C$ is greater than 2, the confidence threshold is reduced to $\frac{1}{C-1}$ which is less than 0.5. Therefore, the STD loss can reduce the confidence threshold when the STD loss converges to 0.
\end{proof}

\begin{proposition}\label{prop:MadryAT-with-CTR-better-than-without}
In the training process of Madry-AT, Free-AT and Fast-AT, the generated adversarial examples with the adversarial objective $L_{SCE}(f(x'),y|\gamma)$ for training are more effective than those generated with $L_{CE}(f(x'),y)$.
\end{proposition}

\begin{proof}
The adversarial examples are generated by maximizing the adversarial objective. In deep neural networks, the adversarial examples are generated by the gradient attack methods, which make the predicted probability decrease on the correct category and increase on the wrong categories. To compare the effectiveness of the adversarial objectives $L_{SCE}$ and $L_{CE}$ generated adversarial examples, the derivation formulas of $L_{SCE}$ and $L_{CE}$ w.r.t. the input $x$ are needed to be calculated.

The derivation formula of $L_{CE}$ w.r.t. the input $x$ is
\begin{subequations}
\begin{numcases}{}
\frac{\partial L_{CE}}{\partial x}=\frac{\partial L_{CE}}{\partial p_{c_1}}\cdot \frac{\partial p_{c_1}}{\partial x}=0\cdot \frac{\partial p_{c_1}}{\partial x}, c_1\ne C \label{eq:CE-wrt-x-wrong}\\
\frac{\partial L_{CE}}{\partial x}=\frac{\partial L_{CE}}{\partial p_C}\cdot \frac{\partial p_C}{\partial x}=-\frac{\ln 2}{p_C}\cdot \frac{\partial p_C}{\partial x}, c_1=C \label{eq:CE-wrt-x-correct}
\end{numcases}
\end{subequations}
where $c_1\ne C$ denotes selecting a wrong category and $c_1=C$ denotes selecting the correct category. In Eq.~\ref{eq:CE-wrt-x-wrong} and \ref{eq:CE-wrt-x-correct}, the coefficient of the gradient $\frac{\partial p_{c_1}}{\partial x}$ is 0 and the coefficient of the gradient $-\frac{\partial p_C}{\partial x}$ is greater than 0, which verify that the purpose of maximizing $L_{CE}$ is to reduce the predicted probability of the correct category. 

The derivation formula of $L_{SCE}$ w.r.t. the input $x$ is 
\begin{subequations}
\begin{numcases}{}
\frac{\partial L_{SCE}}{\partial x}=\frac{\partial L_{SCE}}{\partial p_{c_1}}\cdot \frac{\partial p_{c_1}}{\partial x}=L_{CE}\cdot e^{L_{STD}}\cdot \gamma \cdot\nonumber\\
\frac{\left( L_{STD} \right) ^{-\frac{1}{2}}}{\left( c-2 \right) \cdot \left( c-1 \right)}\cdot \sum_{c=1\land c\ne c_1}^{C-1}{\left( p_{c_1}-p_c \right)}\cdot \frac{\partial p_{c_1}}{\partial x},c_1\ne C \label{eq:SCE-wrt-x-wrong}\\
\frac{\partial L_{SCE}}{\partial x}=\nonumber\\
\frac{\partial L_{SCE}}{\partial p_C}\cdot \frac{\partial p_C}{\partial x}=-\frac{\left( \ln 2 \right) \cdot e^{L_{STD}}}{p_C}\cdot \frac{\partial p_C}{\partial x},c_1=C \label{eq:SCE-wrt-x-correct}
\end{numcases}
\end{subequations}
where $L_{CE}\cdot e^{L_{STD}}\cdot \gamma \cdot\frac{\left( L_{STD} \right) ^{-\frac{1}{2}}}{\left( c-2 \right) \cdot \left( c-1 \right)}$ is greater than 0, $\sum_{c=1\land c\ne c_1}^{C-1}{\left( p_{c_1}-p_c \right)}$ is greater than 0 when $p_{c_1}>u(f(x)\backslash y_x)$, $\sum_{c=1\land c\ne c_1}^{C-1}{\left( p_{c_1}-p_c \right)}$ is less than 0 when $p_{c_1}<u(f(x)\backslash y_x)$, and the greater the probability $p_{c_1}$, the greater the value $\sum_{c=1\land c\ne c_1}^{C-1}{\left( p_{c_1}-p_c \right)}$. The function $u(\cdot)$ denotes calculating the average value of the vector. In Eq.~\ref{eq:SCE-wrt-x-wrong} and \ref{eq:SCE-wrt-x-correct}, the coefficient of the gradient $\frac{\partial p_{c_1}}{\partial x}$ is greater than 0 when $p_{c_1}>u(f(x)\backslash y_x)$, the coefficient of the gradient $\frac{\partial p_{c_1}}{\partial x}$ is less than 0 when $p_{c_1}<u(f(x)\backslash y_x)$, and the greater the probability $p_{c_1}$, the greater the coefficient value of the gradient $\frac{\partial p_{c_1}}{\partial x}$. The coefficient of the gradient $-\frac{\partial p_C}{\partial x}$ is greater than 0. Therefore, the purpose of maximizing $L_{SCE}$ is not only to decrease the predicted probability of the correct category, but also to increase the predicted probability of the wrong category $c_1$ where $p_{c_1}>u(f(x)\backslash y_x)$ and decrease the predicted probability of the wrong category $c_1$ where $p_{c_1}<u(f(x)\backslash y_x)$. In addition, the greater the probability $p_{c_1}$, the greater the weight of the gradient $\frac{\partial p_{c_1}}{\partial x}$.

Based on the proposition introduced in Section~\ref{sec:proposition}, in comparison with the loss $L_{CE}$, the loss $L_{SCE}$ can improve the effectiveness of the generated adversarial examples by increasing the weight of the gradient $\frac{\partial p_{c_1}}{\partial x}$ with higher probability (i.e., $p_{c_1}>u(f(x)\backslash y_x)$, and the greater the probability $p_{c_1}$, the greater the weight of the gradient $\frac{\partial p_{c_1}}{\partial x}$).
\end{proof}

\begin{proposition}\label{prop:TRADES-with-CTR-better-than-without}
In the training process of TRADES, the generated adversarial examples with the adversarial objective $L_{SKL}(f(x'),f(x),y_x|\gamma)$ for training are more effective than those generated with $L_{KL}(f(x'),f(x))$.
\end{proposition}

\begin{proof}
To compare the effectiveness of the adversarial objectives $L_{SKL}$ and $L_{KL}$ generated adversarial examples, the derivation formulas of $L_{SKL}$ and $L_{KL}$ w.r.t. the input $x$ are needed to be calculated.

The derivation formula of $L_{KL}$ w.r.t. the input $x$ is 
\begin{subequations}
\begin{numcases}{}
\frac{\partial L_{KL}}{\partial x}=\frac{\partial L_{KL}}{\partial p_{c_1}}\cdot \frac{\partial p_{c_1}}{\partial x}=-\frac{\ln 2\cdot p_{c_1}^{0}}{p_{c_1}}\cdot \frac{\partial p_{c_1}}{\partial x},c_1\ne C \label{eq:KL-wrt-x-wrong}\\
\frac{\partial L_{KL}}{\partial x}=\frac{\partial L_{KL}}{\partial p_C}\cdot \frac{\partial p_C}{\partial x}=-\frac{\ln 2\cdot p_{C}^{0}}{p_C}\cdot \frac{\partial p_C}{\partial x},c_1=C \label{eq:KL-wrt-x-correct}
\end{numcases}
\end{subequations}
where $p_{c_1}^0$ and $p_C^0$, which are constants, are the initial predicted probability of the initial input $x$. In Eq.~\ref{eq:KL-wrt-x-wrong} and \ref{eq:KL-wrt-x-correct}, the coefficient of the gradient $\frac{\partial p_{c_1}}{\partial x}$ is less than 0 and the coefficient of the gradient $-\frac{\partial p_C}{\partial x}$ is greater than 0. In general, the size of the coefficient of the gradient $\frac{\partial p_{c_1}}{\partial x}$ is much greater than that of the gradient $-\frac{\partial p_C}{\partial x}$, i.e. $\left| -\frac{\ln 2\cdot p_{C}^{0}}{p_C} \right|\gg \left| \frac{\ln 2\cdot p_{c_1}^{0}}{p_{c_1}}\right|$. Therefore, the main purpose of maximizing $L_{KL}$ is to reduce the predicted probability of the correct category.

The derivation formula of $L_{SKL}$ w.r.t. the input $x$ is
\begin{subequations}
\begin{numcases}{}
\frac{\partial L_{SKL}}{\partial x}=\frac{\partial L_{SKL}}{\partial p_{c_1}}\cdot \frac{\partial p_{c_1}}{\partial x}=e^{L_{STD}}\cdot L_{KL}\cdot \gamma\cdot \nonumber\\
\frac{\left( L_{STD} \right) ^{-\frac{1}{2}}}{\left( c-2 \right) \cdot \left( c-1 \right)}\cdot \sum_{c=1\land c\ne c_1}^{C-1}{\left( p_{c_1}-p_c \right)}\cdot \frac{\partial p_{c_1}}{\partial x}+ \nonumber\\
e^{L_{STD}}\cdot \left( -\frac{\ln 2\cdot p_{c_1}^{0}}{p_{c_1}} \right) \cdot \frac{\partial p_{c_1}}{\partial x},c_1\ne C \label{eq:SKL-wrt-x-wrong}\\
\frac{\partial L_{SKL}}{\partial x}= \frac{\partial L_{SKL}}{\partial p_C}\cdot \frac{\partial p_C}{\partial x}=\nonumber\\
e^{L_{STD}}\cdot \left( -\frac{\ln 2\cdot p_{C}^{0}}{p_C} \right) \cdot \frac{\partial p_C}{\partial x},c_1=C \label{eq:SKL-wrt-x-correct}
\end{numcases}
\end{subequations}
where $e^{L_{STD}}\cdot L_{KL}\cdot \gamma\cdot \frac{\left( L_{STD} \right) ^{-\frac{1}{2}}}{\left( c-2 \right) \cdot \left( c-1 \right)}\cdot \sum_{c=1\land c\ne c_1}^{C-1}{\left( p_{c_1}-p_c \right)}$ is regarded as $q_{c_1}$. The value $q_{c_1}$ is greater than 0 when $p_{c_1}>u(f(x)\backslash y_x)$, the value $q_{c_1}$ is less than 0 when $p_{c_1}<u(f(x)\backslash y_x)$, and the greater the probability $p_{c_1}$, the greater the value $q_{c_1}$. As shown in the derivation formulas of $L_{KL}$ and $L_{SKL}$, when $p_{c_1}>u(f(x)\backslash y_x)$, in comparison with the size of the coefficient of the gradient $-\frac{\partial p_C}{\partial x}$, the weight of the coefficient of the gradient $\frac{\partial p_{c_1}}{\partial x}$ in the derivation formula of $L_{SKL}$ is greater than that in the derivation formula of $L_{KL}$, i.e.,
\begin{align}
\frac{\left| q_{c_1}+e^{L_{STD}}\cdot \left( -\frac{\ln 2\cdot p_{c_1}^{0}}{p_{c_1}} \right) \right|}{\left| e^{L_{STD}}\cdot \frac{\ln 2\cdot p_{C}^{0}}{p_C} \right|}>\frac{\left| -\frac{\ln 2\cdot p_{c_1}^{0}}{p_{c_1}} \right|}{\left| \frac{\ln 2\cdot p_{C}^{0}}{p_{C}^{0}} \right|}.
\end{align}

Based on the proposition introduced in Section~\ref{sec:proposition}, in comparison with the loss $L_{KL}$, the loss $L_{SKL}$ can improve the effectiveness of the generated adversarial examples by increasing the weight of the gradient $\frac{\partial p_{c_1}}{\partial x}$ with higher probability (i.e., $p_{c_1}>u(f(x)\backslash y_x)$, and the greater the probability $p_{c_1}$, the greater the weight of the gradient $\frac{\partial p_{c_1}}{\partial x}$).
\end{proof}

\subsection{STD Loss based Adversarial Attacks and Theoretical Analysis}
\label{sec:STD-based-adversarial-attacks}

Due to the preference~\cite{QIFGSM} that adversarial examples are preferred to be classified as the wrong category with higher probability, we propose the STD loss based FGSM, PGD and APGD attack methods (i.e., S-FGSM, S-PGD and S-APGD), which replace the CE loss with the STD loss. 

In this section, we demonstrate that the STD loss based gradient attack methods are better than the CE loss based gradient attack methods in Proposition~\ref{prop:STD-better-than-CE}. We also explain that the STD loss based gradient attack methods are better than the SCE loss and the SKL loss based gradient attack methods. In addition, the reason why the STD loss cannot be directly used as the adversarial objective in adversarial training while the SCE and the SKL losses can be used as the adversarial objective in Madry-AT and TRADES with CTR is mentioned as well.

\begin{proposition}\label{prop:STD-better-than-CE}
In the gradient attack methods, the adversarial examples generated by $L_{STD}$ are more effective than those generated by $L_{CE}$.
\end{proposition}

\begin{proof}
For an input image $x$, the correct category $y_x$ is assumed as the $C$-th category where $C$ is the number of the categories. To compare the effectiveness of the adversarial examples generated by $L_{STD}$ and $L_{CE}$, the derivation formulas of $L_{STD}$ and $L_{CE}$ w.r.t. the input $x$ are needed to be calculated.

The derivation formula of $L_{CE}$ w.r.t. the input $x$ has been calculated in Eq.~\ref{eq:CE-wrt-x-wrong} and \ref{eq:CE-wrt-x-correct}.

The derivation formula of $L_{STD}$ w.r.t. the input $x$ is
\begin{subequations}
\begin{numcases}{}
\frac{\partial L_{STD}}{\partial x}= \nonumber\\
\frac{\partial L_{STD}}{\partial p_{c_1}}\cdot \frac{\partial p_{c_1}}{\partial x}= \frac{\left[ \frac{\sum\nolimits_{c=1}^{C-1}{\left( u\left( f\left( x \right) \backslash y_x \right) -p_c \right) ^2}}{C-2} \right] ^{-\frac{1}{2}}}{\left( C-2 \right) \cdot \left( C-1 \right)}\cdot \nonumber\\
\sum_{c=1\land c\ne c_1}^{C-1}{\left( p_{c_1}-p_c \right)}\cdot \frac{\partial p_{c_1}}{\partial x},c_1\ne C \label{eq:STD-wrt-x-wrong}\\
\frac{\partial L_{STD}}{\partial x}=\frac{\partial L_{STD}}{\partial p_C}\cdot \frac{\partial p_C}{\partial x}=0\cdot \frac{\partial p_C}{\partial x},c_1=C \label{eq:STD-wrt-x-correct}
\end{numcases}
\end{subequations}
where $\frac{\left[ \frac{\sum\nolimits_{c=1}^{C-1}{\left( u\left( f\left( x \right) \backslash y_x \right) -p_c \right) ^2}}{C-2} \right] ^{-\frac{1}{2}}}{\left( C-2 \right) \cdot \left( C-1 \right)}$ is greater than 0, $\sum_{c=1\land c\ne c_1}^{C-1}{\left( p_{c_1}-p_c \right)}$ is greater than 0 when $p_{c_1}>u(f(x)\backslash y_x)$, $\sum_{c=1\land c\ne c_1}^{C-1}{\left( p_{c_1}-p_c \right)}$ is less than 0 when $p_{c_1}<u(f(x)\backslash y_x)$, and the greater the probability $p_{c_1}$, the greater the value $\sum_{c=1\land c\ne c_1}^{C-1}{\left( p_{c_1}-p_c \right)}$. In Eq.~\ref{eq:STD-wrt-x-wrong} and \ref{eq:STD-wrt-x-correct}, the coefficient of the gradient $\frac{\partial p_{c_1}}{\partial x}$ is greater than 0 when $p_{c_1}>u(f(x)\backslash y_x)$, and the coefficient of the gradient $\frac{\partial p_{c_1}}{\partial x}$ is less than 0 when $p_{c_1}<u(f(x)\backslash y_x)$, and the greater the probability $p_{c_1}$, the greater the coefficient value of the gradient $\frac{\partial p_{c_1}}{\partial x}$. The coefficient of the gradient $-\frac{\partial p_C}{\partial x}$ is equal to 0. Therefore, the purpose of maximizing $L_{STD}$ is to increase the predicted probability of the wrong category $c_1$ where $p_{c_1}>u(f(x)\backslash y_x)$ and decrease the predicted probability of the wrong category $c_1$ where $p_{c_1}<u(f(x)\backslash y_x)$. In addition, the greater the probability $p_{c_1}$, the greater the weight of the gradient $\frac{\partial p_{c_1}}{\partial x}$.

Based on the proposition introduced in Section~\ref{sec:proposition}, in comparison with the loss $L_{CE}$, the loss $L_{STD}$ can improve the effectiveness of the generated adversarial examples by increasing the weight of the gradient $\frac{\partial p_{c_1}}{\partial x}$ with higher probability (i.e., $p_{c_1}>u(f(x)\backslash y_x)$, and the greater the probability $p_{c_1}$, the greater the weight of the gradient $\frac{\partial p_{c_1}}{\partial x}$).
\end{proof}

Because the coefficient of the gradient $\frac{\partial p_C}{\partial x}$ is equal to 0 in Eq.~\ref{eq:STD-wrt-x-correct}, the weight of the gradient $\frac{\partial p_{c_1}}{\partial x}$ with higher probability in the derivation formulas of $L_{STD}$ is greater than that in the derivation formulas of $L_{SCE}$ and $L_{SKL}$. Therefore, based on the proposition introduced in Section~\ref{sec:proposition}, the STD loss based gradient attack methods are more effective than the SCE loss based and the SKL loss based.

\textit{Why is not the STD loss directly used as the adversarial objective in adversarial training, but the SCE and the SKL losses can be used as the adversarial objective in Madry-AT and TRADES with CTR, respectively?} There are two reasons. First, adversarial training is a min-max value problem~\cite{PGD} where the function formulation of the adversarial objective should be equal to or included in the function formulation of the training objective~\cite{PGD,Free-AT,Fast-AT,TRADES}. Hence, the adversarially trained model will be at an equilibrium point, rather than partial to one of the minimization and maximization in the minimax problem, resulting in the imbalance. For Madry-AT~\cite{PGD}, Free-AT~\cite{Free-AT} and Fast-AT~\cite{Fast-AT}, the function formulation of the adversarial objective is equal to that of the training objective. That is, if the adversarial objective is the STD loss, the training objective is also the STD loss. An extreme case that the model does not learn effective correct category features (i.e., the predicted probability of each category is the same) may occur when the STD loss (i.e., the training objective) converges. Second, to keep the characteristics of the original adversarial training, we only add the STD loss to the original adversarial objective and training objective of Madry-AT~\cite{PGD}, Free-AT~\cite{Free-AT}, Fast-AT~\cite{Fast-AT} and TRADES~\cite{TRADES} (i.e., the SCE and SKL losses), rather than replace them with the STD loss.

\section{Experiments}
\label{sec:experiments}
This section evaluates the effectiveness of our CTR on natural training and adversarial training, namely the MDL loss and the STD loss on natural training and adversarial training, respectively. In addition, the effectiveness of the STD loss (and its variants) based gradient attack methods is evaluated as well. The experiments are executed on four NVIDIA V100 GPUs for natural training and four NVIDIA 2080ti GPUs for adversarial training, and the deep learning framework used is PyTorch 3.7. Note that the percentage sign (\%) is omitted in the accuracies and attack success rates (ASR) of all evaluation results.

\begin{table*}[t]
\begin{center}
\caption{Comparison between our method and the state-of-the-art approaches on accuracy (\%) and average rank. Note that M, FM, C10 and C100 represent MNIST, FashionMNIST, CIFAR10 and CIFAR100 respectively.}
\label{tab:natural-training-comparison}
\begin{tabular}{cccccccccccccl}
\hline
\multirow{3}{*}{Methods} & \multicolumn{4}{c}{MNISTNet}                                      & \multicolumn{4}{c}{VGG16}                                         & \multicolumn{4}{c}{ResNet50}                                      & \multirow{3}{*}{\begin{tabular}[c]{@{}l@{}}Avg.\\ rank\end{tabular}} \\
                        & \multicolumn{2}{c}{Cyclic}      & \multicolumn{2}{c}{Multistep}   & \multicolumn{2}{c}{Cyclic}      & \multicolumn{2}{c}{Multistep}   & \multicolumn{2}{c}{Cyclic}      & \multicolumn{2}{c}{Multistep}   &                                                                      \\
                        & M              & FM             & M              & FM             & C10            & C100           & C10            & C100           & C10            & C100           & C10            & C100           &                                                                      \\ \hline
DO\cite{Dropout}                 & 99.56          & 93.41          & 99.61          & 93.73          & 93.05          & 70.67          & 93.45          & 72.16          & 92.69          & 70.67          & 94.53          & 76.04          & 4.33                                                                 \\
FL\cite{Focal-Loss}                      & 99.49          & 93.11          & 99.55          & 93.18          & 93.14          & 71.37          & 92.92          & 71.36          & 91.68          & 69.62          & 92.12          & 74.06          & 6.08                                                                 \\
DB\cite{Dropblock}               & 99.45          & 92.58          & 99.53          & 93.46          & 93.46          & 71.12          & 93.52          & 71.58          & 92.99          & 66.05          & 94.57          & 75.96          & 5.25                                                                 \\
MSD\cite{Multi-Sample-Dropout}                     & 99.55          & 92.89          & 99.58          & 93.36          & 93.31          & 71.39          & 93.59          & 72.17          & 92.74          & 70.04          & 94.66          & 75.42          & 4.0                                                                  \\
Disout\cite{Beyond-Dropout}                  & 99.54          & 92.15          & 99.51          & 92.45          & 93.07          & 71.48          & 93.57          & 71.42          & 93.00          & 67.47          & 94.62          & 76.00          & 5.25                                                                 \\
LR\cite{Label-Relaxation}                      & 99.38          & 92.77          & 99.45          & 93.00          & 92.39          & 70.30          & 88.42          & 71.10          & 92.31          & 69.32          & 92.29          & 73.80          & 7.42                                                                 \\
MDL (ours)                     & 99.61          & \textbf{93.53} & \textbf{99.65} & \textbf{93.95} & \textbf{93.51} & \textbf{71.99} & \textbf{93.67} & \textbf{72.59} & \textbf{93.99} & \textbf{72.05} & 94.69          & \textbf{76.75} & \textbf{1.17}                                                        \\
MDL-PCC (ours)                 & \textbf{99.62} & 93.49          & 99.64          & 93.81          & 93.35          & 71.45          & 92.96          & 72.03          & 93.54          & 71.86          & \textbf{94.89} & 76.56          & 2.5                                                                  \\ \hline
\end{tabular}
\end{center}
\end{table*}
\begin{table*}[t]
\begin{center}
\caption{Compatibility verification of our method with dropout variant, label smoothing and its variant on accuracy(\%). Note that M, FM, C10 and C100 represent MNIST, Fashion-MNIST, CIFAR10 and CIFAR100 respectively.}
\label{tab:natural-training-compatibility-comparison}
\begin{tabular}{ccccccccccccc}
\hline
\multirow{3}{*}{Methods} & \multicolumn{4}{c}{MNISTNet}                                      & \multicolumn{4}{c}{VGG16}                                         & \multicolumn{4}{c}{ResNet50}                                      \\
                        & \multicolumn{2}{c}{Cyclic}      & \multicolumn{2}{c}{Multistep}   & \multicolumn{2}{c}{Cyclic}      & \multicolumn{2}{c}{Multistep}   & \multicolumn{2}{c}{Cyclic}      & \multicolumn{2}{c}{Multistep}   \\
                        & M              & FM             & M              & FM             & C10            & C100           & C10            & C100           & C10            & C100           & C10            & C100           \\ \hline
DC\cite{Dropconnect}                      & 99.50          & 92.18          & 99.50          & 92.60          & 93.08          & 71.39          & \textbf{93.52} & 71.55          & 92.89          & 67.27          & 94.50          & 76.01          \\
DC+MDL (ours)           & \textbf{99.55} & \textbf{92.80} & \textbf{99.56} & \textbf{93.21} & \textbf{93.55} & \textbf{71.86} & 93.33          & \textbf{72.11} & \textbf{94.47} & \textbf{72.98} & \textbf{94.55} & \textbf{77.11} \\ \hline
LS\cite{Label-Smoothing}                      & 99.47          & 93.43          & 99.55          & 93.53          & 93.17          & 72.38          & \textbf{93.49} & 72.39          & 92.51          & \textbf{69.19} & 93.99          & \textbf{76.15} \\
LS+MDL (ours)           & \textbf{99.48} & \textbf{93.55} & \textbf{99.55} & \textbf{94.05} & \textbf{93.37} & \textbf{72.41} & 93.42          & \textbf{72.46} & \textbf{93.36} & 68.34          & \textbf{94.58} & 76.05          \\ \hline
OLS\cite{Online-Label-Smoothing}                     & \textbf{99.55} & 93.37          & 99.57          & 93.75          & 93.26          & 72.51          & \textbf{93.69} & 72.35          & 92.55          & \textbf{70.23} & 94.47          & 76.00          \\
OLS+MDL (ours)          & 99.52          & \textbf{93.53} & \textbf{99.58} & \textbf{93.88} & \textbf{93.34} & \textbf{72.91} & 93.44          & \textbf{72.72} & \textbf{93.42} & 69.29          & \textbf{94.63} & \textbf{76.42} \\ \hline
\end{tabular}
\end{center}
\end{table*}
\begin{table}[t]
\begin{center}
\caption{The average accuracy (\%) and build time (mins) of three runs. The type of DNN is MNISTNet on FashionMNIST and VGG16 on CIFAR10. The learning rate strategy is Cyclic.}
\label{tab:average-accuracy-for-three-times}
\begin{tabular}{ccccc}
\hline
            & \multicolumn{2}{c}{FashionMNIST}  & \multicolumn{2}{c}{CIFAR10}       \\ \hline
Methods     & Accs                & Time & Accs                & Time \\ \hline
Dropout~\cite{Dropout}     & 93.38+0.13          & 9.1         & 93.16+0.11          & 144         \\
FL~\cite{Focal-Loss}          & 93.01+0.13          & 9.7         & 92.85+0.25          & 161         \\
Dropblock~\cite{Dropblock}   & 92.67+0.12          & 9.5         & 93.27+0.17          & 158         \\
MSD~\cite{Multi-Sample-Dropout}         & 92.84+0.08          & 11.4        & 93.18+0.27          & 191         \\
Disout~\cite{Beyond-Dropout}      & 92.02+0.13          & 9.3         & 93.02+0.14          & 159         \\
LR~\cite{Label-Relaxation}          & 92.88+0.11          & 10.0        & 92.42+0.10          & 163         \\
DC~\cite{Dropconnect} & 92.00+0.16          & 10.7        & 93.08+0.04          & 144         \\
LS~\cite{Label-Smoothing}          & 93.43+0.15          & 9.3         & 93.07+0.11          & 159         \\
OLS~\cite{Online-Label-Smoothing}         & 93.42+0.11          & 9.8         & 93.09+0.28          & 160         \\
MDL (ours)         & \textbf{93.51+0.05} & 19.4        & \textbf{93.41+0.11} & 193         \\ \hline
\end{tabular}
\end{center}
\end{table}
\begin{figure*}[t]
\centering
\includegraphics[width=0.99\textwidth]{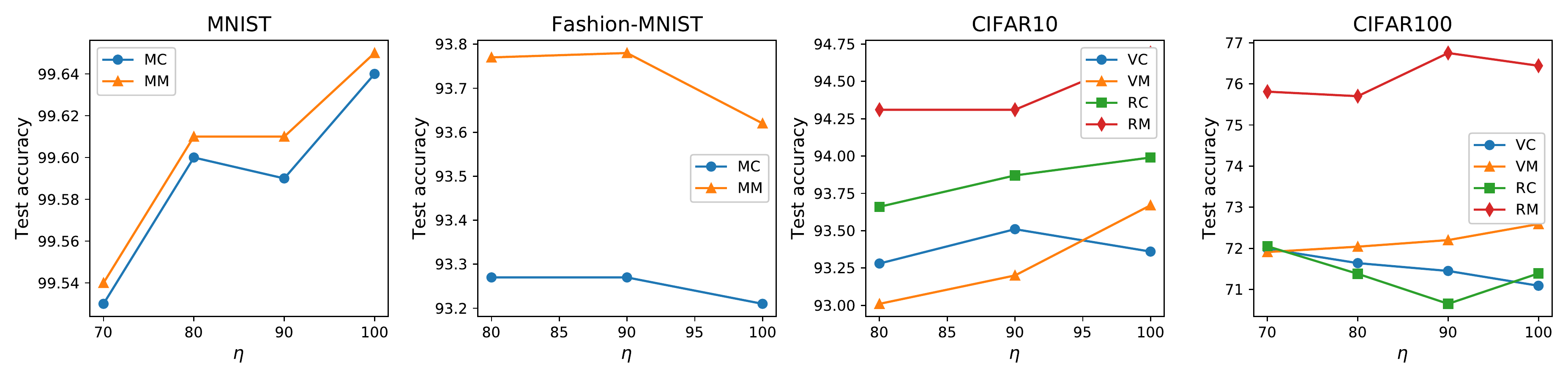} 
\caption{Relationship between the parameter $\eta$ and accuracy (\%). Note that MC, MM, VC, VM, RC and RM represent MNIST-Net with Cyclic and Multistep, VGG16 with Cyclic and Multistep, ResNet50 with Cyclic and Multistep, respectively.}
\label{fig:diff-eta}
\end{figure*}
\begin{figure*}[t]
\centering
\includegraphics[width=0.99\textwidth]{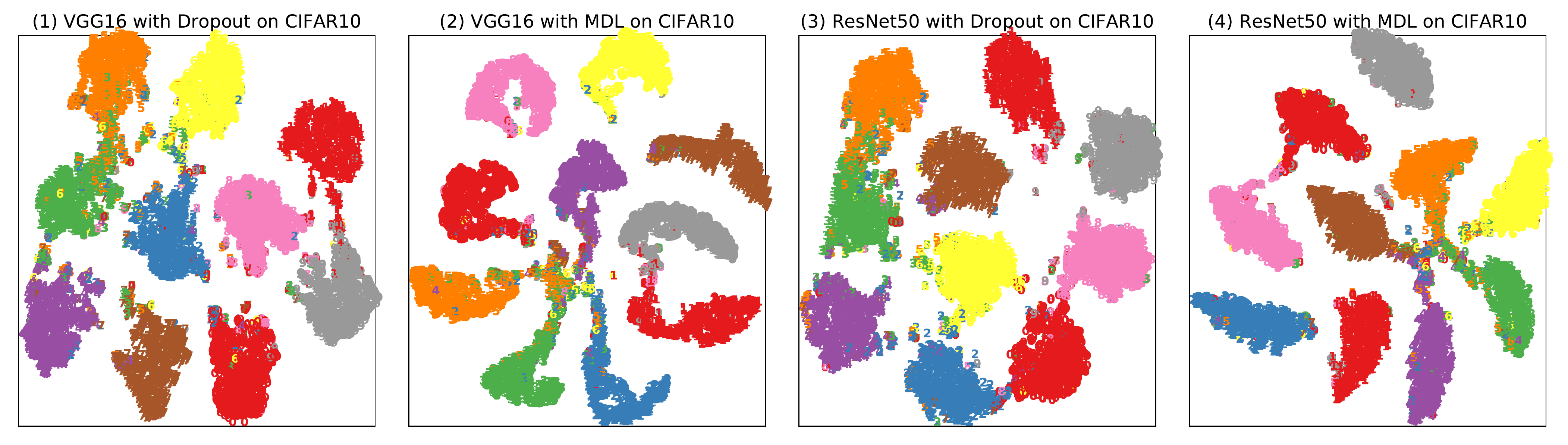} 
\caption{The t-SNE visualization of the model logits on CIFAR10. The subgraph (1), (2), (3) and (4) are VGG16 and ResNet50 without and with MDL, respectively.}
\label{fig:tSNE-of-model-logits}
\end{figure*}
\begin{figure*}[t]
\centering
\includegraphics[width=0.99\textwidth]{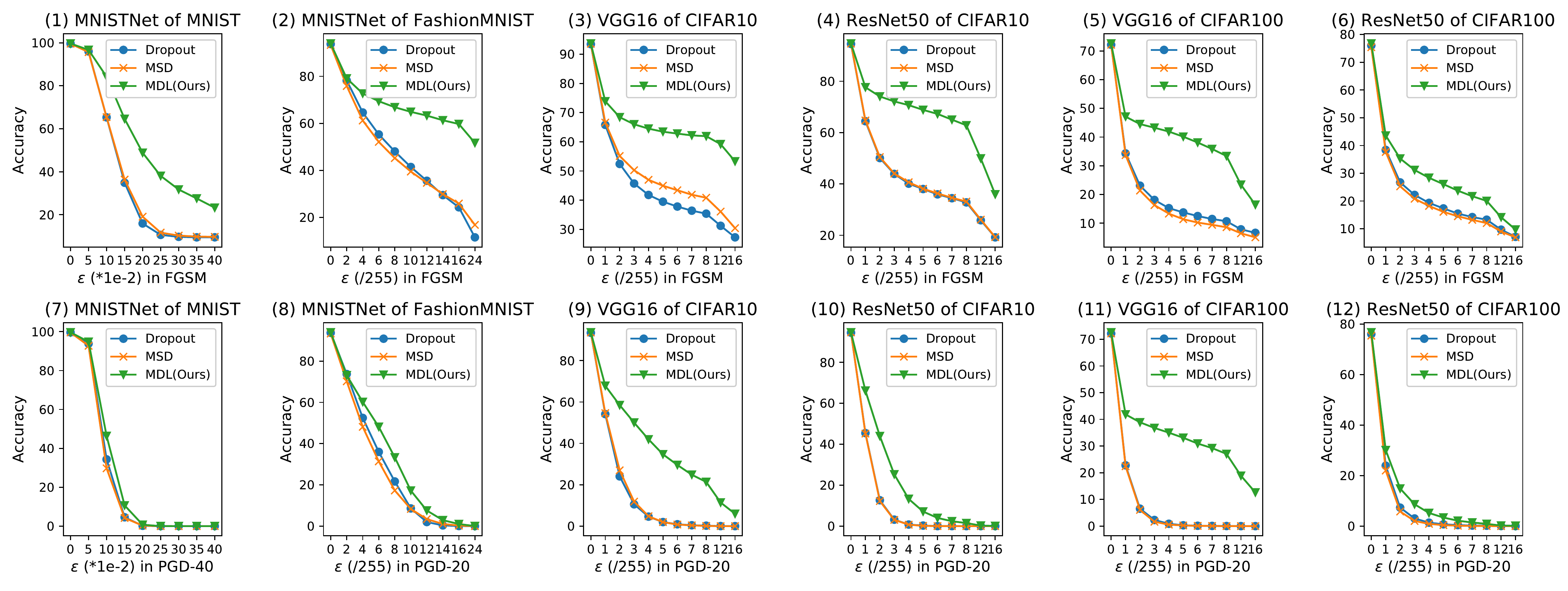} 
\caption{The accuracy (\%) lines of dropout, MSD and our MDL on 4 datasets under FGSM and PGD attacks with different strength, respectively. The first row of subgraphs are under FGSM attack, the second row are under PGD attack.}
\label{fig:dropout-natural-lines}
\end{figure*}
\begin{figure}[t]
\centering
\includegraphics[width=0.99\columnwidth]{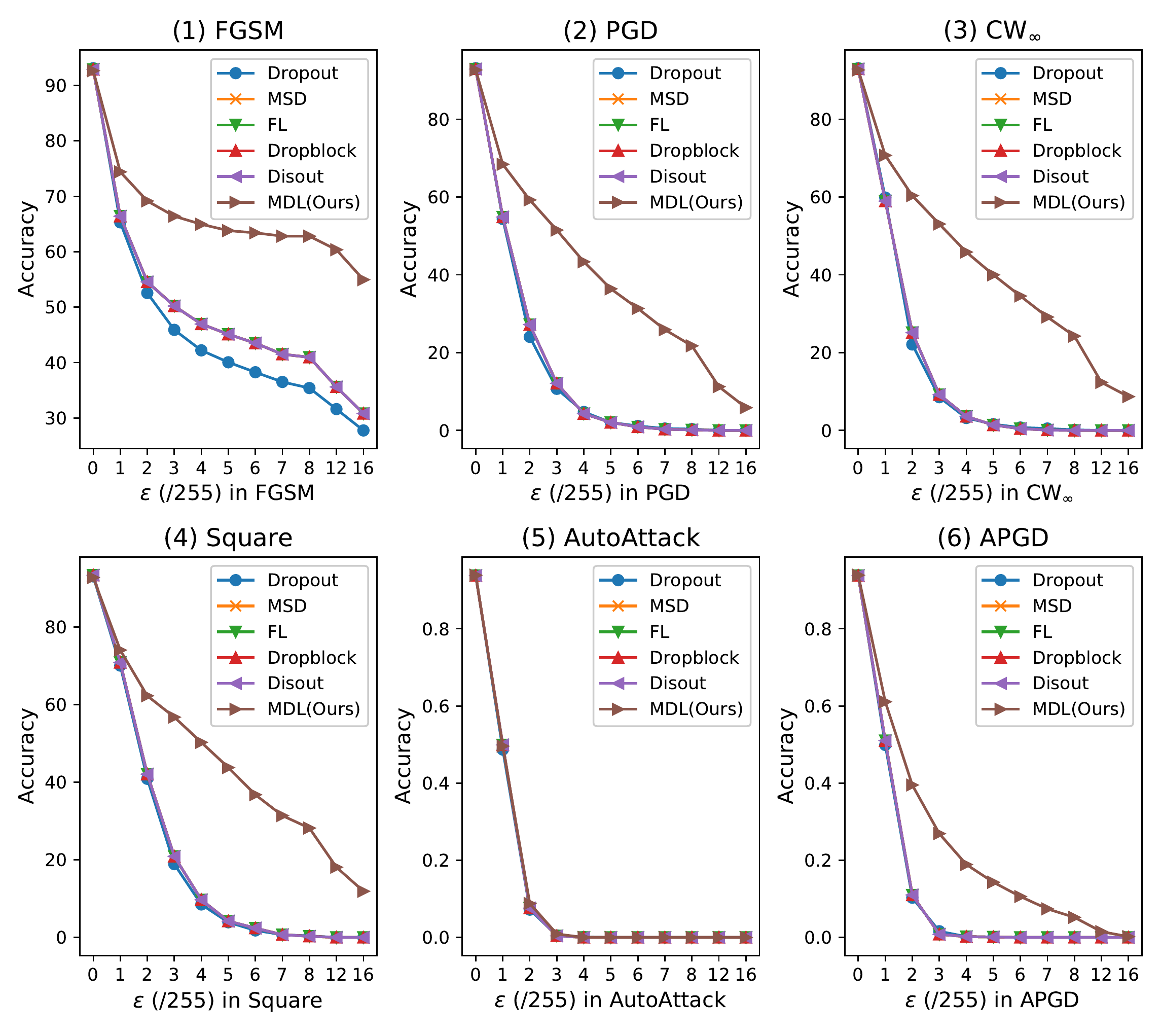} 
\caption{The accuracy (\%) lines of more regularization methods on CIFAR10 under more attacks with different strength, respectively. The type of DNN is VGG16.}
\label{fig:more_attacks_natural_lines}
\end{figure}
\begin{figure*}[t]
\centering
\includegraphics[width=0.99\textwidth]{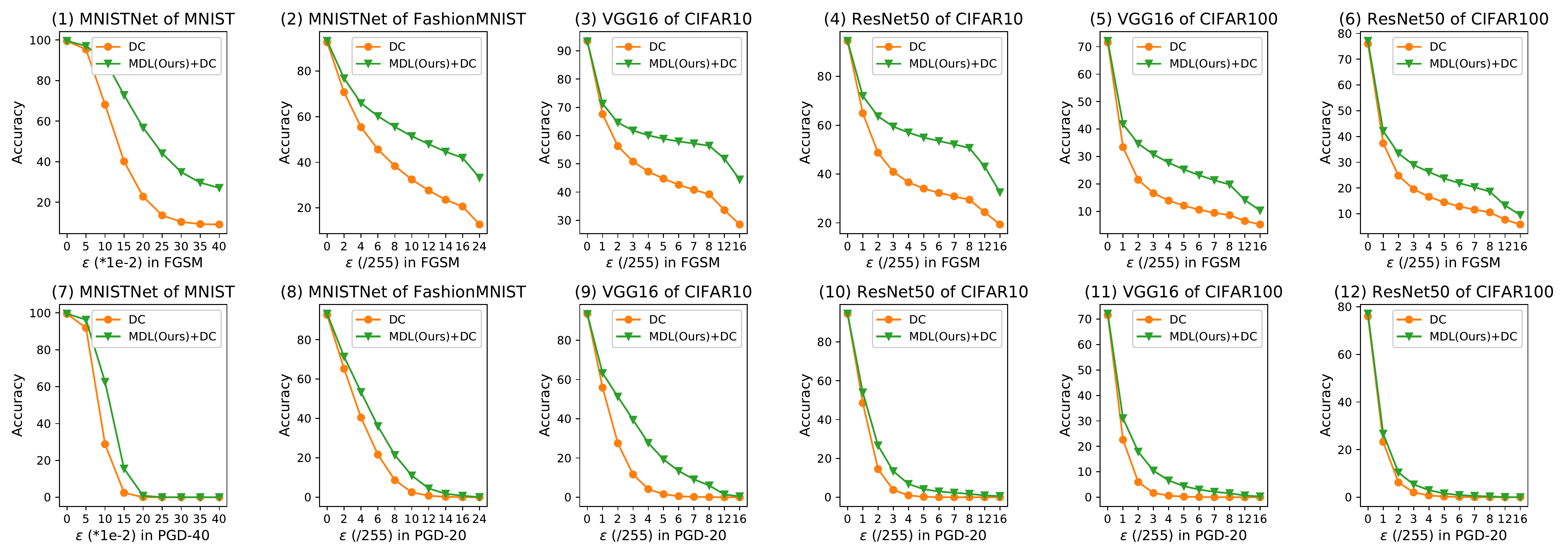} 
\caption{The accuracy (\%) lines of DC and the combination of both MDL and DC on 4 datasets under FGSM and PGD attacks with different strength, respectively. The first row of subgraphs are under FGSM attack, the second row are under PGD attack.}
\label{fig:dropconnect-natural-lines}
\end{figure*}
\begin{table}[t]
\begin{center}
\caption{The number of wrong categories with probability less than the theoretical CT (i.e. $\frac{1}{C-1}$) on test set. FMNIST denotes FashionMNIST.}
\label{tab:statistical-analysis-of-CT}
\begin{tabular}{ccccc}
\hline
Methods   & MNIST          & FMNIST   & CIFAR10        & CIFAR100        \\ \hline
Dropout~\cite{Dropout}   & 89858          & 88545          & 88987          & 977710          \\
FL~\cite{Focal-Loss}        & 89608          & 86465          & 88465          & 966173          \\
Dropblock~\cite{Dropblock} & 89867          & 88566          & 88972          & 974991          \\
MSD~\cite{Multi-Sample-Dropout}       & 89860          & 88680          & 88993          & 976665          \\
Disout~\cite{Beyond-Dropout}    & 89855          & 88750          & 88952          & 975592          \\
LR~\cite{Label-Relaxation}        & 89834          & 89009          & 88773          & 965481          \\
DC~\cite{Dropconnect}        & 89858          & 88764          & 89002          & 975522          \\
LS~\cite{Label-Smoothing}        & 89793          & 88293          & 88951          & 968836          \\
OLS~\cite{Online-Label-Smoothing}       & 89809          & 88346          & 88977          & 969770          \\
MDL (ours) & \textbf{89942} & \textbf{89028} & \textbf{89140} & \textbf{984153} \\ \hline
\end{tabular}
\end{center}
\end{table}
\begin{figure}[t]
\centering
\includegraphics[width=0.99\columnwidth]{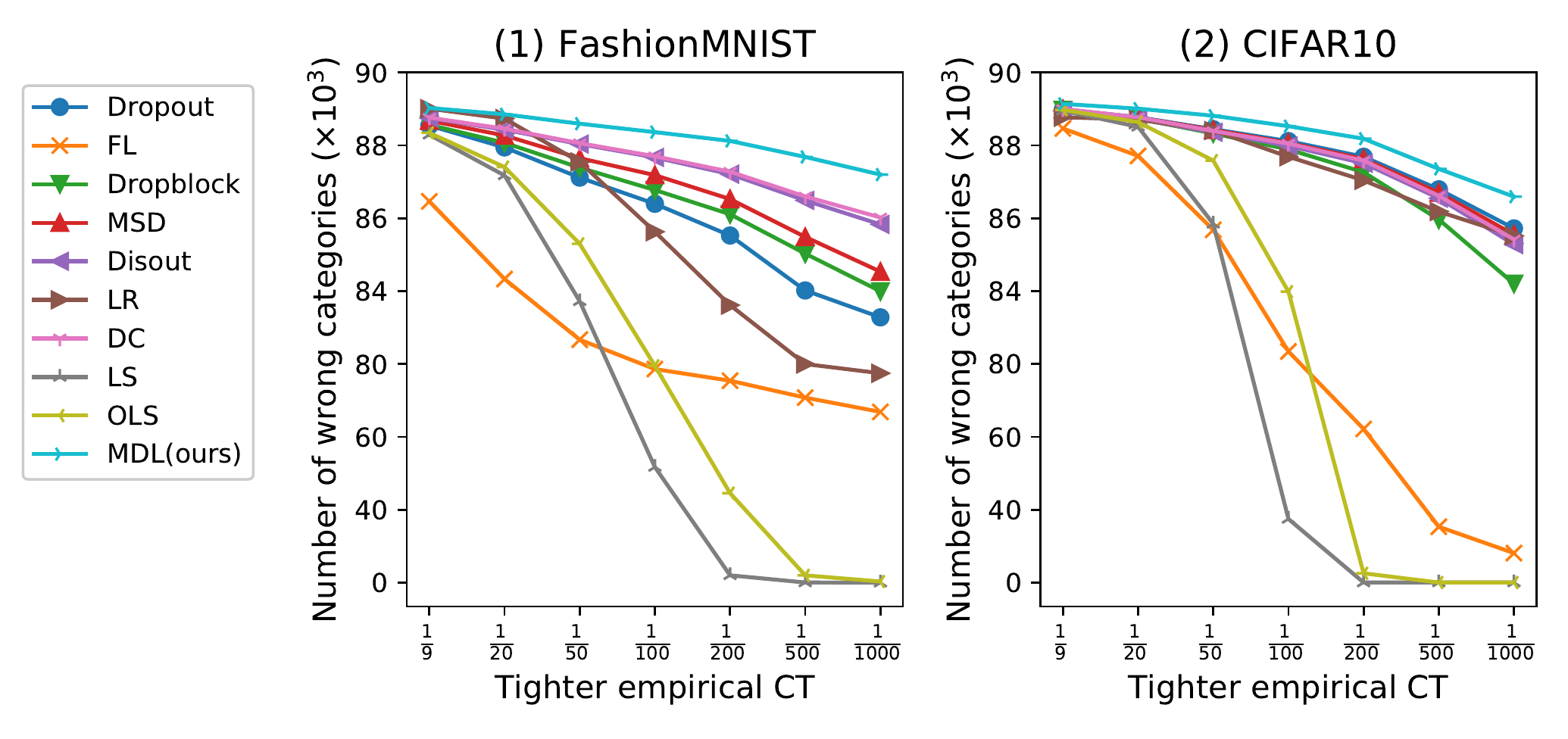} 
\caption{The exploration of the tighter empirical CT on FashionMNIST and CIFAR10 test set. The y-axis denotes the number of wrong categories with probability less than the empirical CT of the x-axis.}
\label{fig:tighter-empirical-CT}
\end{figure}

\subsection{Natural Training with CTR}
\label{sec:natural-training-with-CTR}
The setup is firstly introduced in this section. Then the generalization and robustness evaluations of the MDL loss are analyzed for natural training.

\subsubsection{Setup}
\label{sec:setup1}
The effectiveness of the proposed mask-guided divergence loss function is evaluated with MNIST~\cite{MNIST}, FashionMNIST~\cite{Fashion-MNIST}, CIFAR10~\cite{CIFAR} and CIFAR100~\cite{CIFAR} datasets for natural training. MNISTNet~\cite{MNIST-Net}, which is a 7-layer network with 4 convolutional layers and 3 fully connected  layers, is used to evaluate the generalization and robustness of the MDL on MNIST and FashionMNIST. VGG16~\cite{VGG16} and ResNet50~\cite{ResNet} are used to evaluate the generalization and robustness of the MDL on CIFAR10 and CIFAR100.

All DNNs are trained with a batch size of 128 using stochastic gradient descent (SGD) with 0.9 momentum and $5\cdot 10^{-4}$ decay rate. In natural training, MNISTNet trains 20 epochs on MNIST and 50 epochs on FashionMNIST, and both VGG16 and ResNet50 train 200 epochs on CIFAR10 and CIFAR100, respectively. 

In addition, two popular learning rate schedulers are used: Multistep and Cyclic~\cite{Cyclic}. For Multistep, the learning rate is initially set as 0.01 and decays at the $\frac{1}{2}$ and $\frac{3}{4}$ of the total epochs by a factor of 0.1. For Cyclic, the max and min learning rates are set as 0.01 and 0, respectively. The metric of all evaluations is the accuracy (\%). In the evaluation of our MDL, the number of sub-DNN $K$ in each training step is set as 4, the dropout rate is 0.5, $\eta$ choose from $\left\{100, 90, 80, 70\right\}$ and the weight coefficient $\rho$ is 1.

\textbf{Generalization Evaluation Setup.} The 9 compared methods in the evaluation are consist of 5 model-level methods (i.e., dropout (DO)~\cite{Dropout}, dropconnect (DC)~\cite{Dropconnect}, dropblock (DB)~\cite{Dropblock}, multi-sample dropout (MSD)~\cite{Multi-Sample-Dropout}, disout~\cite{Beyond-Dropout}) and 4 loss-level methods (i.e., label smoothing (LS)~\cite{Label-Smoothing}, focal loss (FL)~\cite{Focal-Loss}, label relaxation (LR)~\cite{Label-Relaxation}, online label smoothing (OLS)~\cite{Online-Label-Smoothing}). The training parameters of these compared methods are selected from the original experiment. In addition, the cosine similarity based diversity $D_1(\boldsymbol{P}_s^1, \boldsymbol{P}_s^2)$ in Eq.~\ref{eq:CS-based-diversity} replaces with Pearson correlation coefficient (PCC) based diversity $D_2(\boldsymbol{P}_s^1, \boldsymbol{P}_s^2)$, called MDL-PCC:
\begin{align}\label{eq:PCC}
Pe(\boldsymbol{P}_s^1, \boldsymbol{P}_s^2) =\frac{\sum_{i=1}^{M-1}(\bar{p}_s^{1i}-u(\boldsymbol{\bar{P}}_s^1))(\bar{p}_s^{2i}-u(\boldsymbol{\bar{P}}_s^2))}{\Vert \boldsymbol{\bar{P}}_s^1-u(\boldsymbol{\bar{P}}_s^1)\Vert_2\Vert \boldsymbol{\bar{P}}_s^2-u(\boldsymbol{\bar{P}}_s^2)\Vert_2}
\end{align}
\begin{align}\label{eq:PCC-based-diversity}
D_2(\boldsymbol{P}_s^1, \boldsymbol{P}_s^2) = \left(Pe(\boldsymbol{P}_s^1, \boldsymbol{P}_s^2)+1\right)/2
\end{align}
where $\bar{p}_s^{1i}$ and $\bar{p}_s^{2i}$ are the $i^{th}$ element in $\boldsymbol{\bar{P}}_s^1$ and $\boldsymbol{\bar{P}}_s^2$, $\Vert\cdot\Vert_2$ is \textit{l2} norm, $u(\cdot)$ is the average of vector. The smaller $D_2(\boldsymbol{P}_s^1, \boldsymbol{P}_s^2)$, the greater the diversity.

\textbf{Robustness Evaluation Setup.} FGSM~\cite{FGSM}, PGD~\cite{PGD}, APGD~\cite{AutoAttack}, CW$_\infty$~\cite{CW}, Square~\cite{Square} and AutoAttack (AA)~\cite{AutoAttack} with different attack strength, e.g., $\varepsilon=[1/255,2/255,\cdots,8/255,12/255,16/255]$ on CIFAR10 and CIFAR100, $\varepsilon=[0.05,0.1,\cdots,0.3,0.35,0.4]$ on MNIST and $\varepsilon=[2/255,4.255,\cdots,16/255,24/255]$ on FashionMNIST, are used to evaluate the robustness of our MDL for natural training. Note that PGD with 40 steps (i.e., PGD-40) on MNIST and PGD with 20 steps (i.e., PGD-20) on FashionMNIST, CIFAR10, and CIFAR100 are used. On CIFAR10, the max iteration and learning rate of CW$_\infty$ are 10 and 0.01, respectively, the max number of queries of Square is 5000, the max iteration of APGD is 100, and the version of AutoAttack is standard. The first 2000 inputs of the testset are used as the evaluation set on white-box attacks except for AutoAttack, and the first 1000 inputs of the testset are used as the evaluation set on black-box attacks and AutoAttack.

\subsubsection{Generalization Evaluation of MDL}
\label{sec:generalization-evaluation-MDL}
\textbf{Comparison among MDL, Dropout, Label Smoothing and Their Variants.}
Tables~\ref{tab:natural-training-comparison} and \ref{tab:natural-training-compatibility-comparison} show the accuracies of MDL, MDL-PCC and the other 9 compared regularization methods. In Table~\ref{tab:natural-training-comparison}, the average rank of MDL is the best, and MDL-PCC is the second, which shows that the improved diversity of the DNN can improve the generalization performance, and the cosine similarity based diversity is better than the Pearson correlation coefficient based diversity. As shown in Tables~\ref{tab:natural-training-comparison} and \ref{tab:natural-training-compatibility-comparison}, the improvements of MDL compared with the maximum accuracy of the 9 regularization methods are 0.05\%, 0.1\%, 0.05\%, 0.99\%, 1.38\% on MNIST, FashionMNIST, CIFAR10 (both VGG16 and ResNet50 of CIFAR10), and CIFAR100 (ResNet50 of CIFAR100) by using the Cyclic scheduler, respectively, and 0.04\%, 0.2\%, 0.03\%, 0.2\%, 0.6\% on MNIST, FashionMNIST, CIFAR10 (both VGG16 and ResNet50 of CIFAR10), and CIFAR100 (ResNet50 of CIFAR100) by using the Multistep scheduler, respectively. Although OLS has the maximum accuracy on VGG16 of CIFAR100 by using the Cyclic scheduler, the combination of OLS and MDL can improve by 0.4\% to OLS. In summary, MDL has the best comprehensive generalization.

\textbf{The Influence of Random Seed for Test Accuracy.} Table~\ref{tab:average-accuracy-for-three-times} shows that the average test accuracy of MDL is better than that of the other 9 regularization methods on FashionMNIST and CIFAR10.

\textbf{Compatibility between MDL and Other Methods.}
Table~\ref{tab:natural-training-compatibility-comparison} shows the generalization improvements of MDL combined with dropconnect, LS, and OLS on 4 datasets, respectively. For the DNN with dropconnect, MDL can effectively improve the accuracies of the DNNs on 4 datasets and 2 learning schedulers, i.e., 0.05\%, 0.62\%, 0.47\%, 1.58\%, 0.47\%, and 5.71\% on MNIST, FashionMNIST, CIFAR10 (both VGG16 and ResNet50 of CIFAR10) and CIFAR100 (both VGG16 and ResNet50 of CIFAR100) respectively by using the Cyclic scheduler, and 0.06\%, 0.61\%, 0.05\%, 0.56\% and 1.1\% on MNIST, FashionMNIST, CIFAR10 (ResNet50 of CIFAR10), CIFAR100 (both VGG16 and ResNet50 of CIFAR100) respectively by using the Multistep scheduler, except for the accuracy reduction of 0.19\% on VGG16 of CIFAR10 by using the Multistep scheduler. For label smoothing and online label smoothing, MDL can improve the accuracies of the DNNs on MNIST and Fashion-MNIST datasets, and the accuracies of the most DNNs on CIFAR10 and CIFAR100. Therefore, MDL has good compatibility with dropconnect (a variant of dropout) and certain compatibility with label smoothing and online label smoothing (a variant of label smoothing).

\textbf{Impacts of $\eta$ on Generalization.}
Fig.~\ref{fig:diff-eta} shows the impacts of the size of $\eta$ on 4 datasets. As Fig.~\ref{fig:diff-eta} shows, the parameter $\eta$ has a certain impact on the accuracies of all DNNs on 4 datasets. Generally, when the size of $\eta\%$ is equal or near to the percentage of the accuracy, the DNN using MDL can obtain the optimal or sub-optimal generalization, which saves the time consumption on adjusting the parameter $\eta$. For example, for MNIST, FashionMNIST, CIFAR10, and CIFAR100, when $\eta\%$ is set to $100\%$, $90\%$, $100\%$, and $70\%$ near the percentage of the accuracies, respectively (i.e., $99\%$, $93\%$, $94\%$, and $72\%$), all DNNs can obtain the optimal generalization. Therefore, the optimal generalization of the DNN can be achieved by setting the parameter $\eta\%$ as the percentage of the accuracy approximately. Note that the initial value of the parameter $\eta$ is set to $100$.

\textbf{The t-SNE Visualization of Features Extracted from the DNNs Trained by Different Methods.}
To study the influence of MDL on feature extraction, we visualize the model logits of the CIFAR10 test dataset by using t-SNE~\cite{t-SNE}. Fig.~\ref{fig:tSNE-of-model-logits}(1) and Fig.~\ref{fig:tSNE-of-model-logits}(2) are the t-SNE visualization of VGG16 with Dropout and MDL, respectively. Fig.~\ref{fig:tSNE-of-model-logits}(3) and Fig.~\ref{fig:tSNE-of-model-logits}(4) are the t-SNE visualization ResNet50 with Dropout and MDL, respectively. In Fig.~\ref{fig:tSNE-of-model-logits}(1) and Fig.~\ref{fig:tSNE-of-model-logits}(3), there are many separate small clusters, which verifies that the DNN with dropout has learned more complex classification boundaries. However, in Fig.~\ref{fig:tSNE-of-model-logits}(2) and Fig.~\ref{fig:tSNE-of-model-logits}(4), due to a few separate small clusters, the classification boundaries of the DNN with MDL are simple. Therefore, dropout with MDL can achieve better generalization. Note that because MDL can ensure that the probability of each false category is less than $\frac{1}{M-1}$, the probabilities of all categories are close around $\frac{1}{M-1}$ for wrong classification samples and low confidence samples, thereby resulting in a radial distribution in Fig.~\ref{fig:tSNE-of-model-logits}(2) and Fig.~\ref{fig:tSNE-of-model-logits}(4).

\begin{table*}[t]
\begin{center}
\caption{The accuracies (\%) of TRADES~\cite{TRADES} with or without CTR on each attacks. Time denotes build time (mins). Bold indicates higher performance.}
\label{tab:accuracies-of-TRADES}
\begin{tabular}{ccccccccccccc}
\hline
                               &      & \multicolumn{1}{c|}{}     & \multicolumn{5}{c|}{Multistep}                                                               & \multicolumn{5}{c}{Cyclic}                                               \\ \hline
Methods                        & $\beta$ & \multicolumn{1}{c|}{$\gamma$} & Natural        & FGSM           & PGD            & AA            & \multicolumn{1}{c|}{Time} & Natural        & FGSM           & PGD            & AA            & Time  \\ \hline
TRADES w/o CTR~\cite{TRADES}                & 1    & 0                         & 86.67          & 54.5           & 46.75          & 41.9          & 627.9                     & 86.23          & 54.7           & 47.4           & 42.9          & 633.8 \\ \cline{1-1}
\multirow{9}{*}{TRADES w/ CTR (ours)} & 1    & 1                         & 86.85          & 56.6           & 47.65          & 41.7          & 667.0                     & 86.52          & 55.15          & 48.3           & 42.7          & 627.5 \\
                               & 1    & 2                         & 87.11          & 55.56          & 47.85          & 42.6          & 646.9                     & 87.0           & 55.05          & 46.8           & 41.3          & 646.3 \\
                               & 1    & 3                         & 87.89          & 55.55          & 47.9           & 41.3          & 628.5                     & 87.48          & 55.3           & 46.7           & 40.1          & 625.6 \\
                               & 1    & 4                         & 88.80          & 51.4           & 42.8           & 36.4          & 648.3                     & 87.65          & 55.5           & 47.55          & 40.7          & 656.4 \\
                               & 1    & 5                         & 88.27          & 49.95          & 39.3           & 28.8          & 647.2                     & 90.02          & 49.3           & 36.65          & 28.8          & 620.2 \\
                               & 2    & 3                         & \textbf{86.76} & \textbf{57.2}  & \textbf{49.35} & \textbf{43.6} & 640.3                     & \textbf{86.37} & \textbf{57.15} & \textbf{49.95} & \textbf{44.9} & 626.7 \\
                               & 3    & 3                         & \textbf{86.11} & \textbf{58.15} & \textbf{51.35} & \textbf{44.6} & 649.4                     & \textbf{85.25} & \textbf{58.4}  & \textbf{51.35} & \textbf{45.3} & 623.2 \\
                               & 4    & 3                         & \textbf{85.91} & \textbf{59.65} & \textbf{52.75} & \textbf{46.1} & 639.8                     & \textbf{85.08} & \textbf{58.15} & \textbf{52.65} & \textbf{46.0} & 633.1 \\
                               & 5    & 3                         & \textbf{85.18} & \textbf{58.55} & \textbf{52.15} & \textbf{46.8} & 645.8                     & \textbf{85.23} & \textbf{58.4}  & \textbf{53.35} & \textbf{46.9} & 622.2 \\ \cline{1-1}
TRADES w/o CTR~\cite{TRADES}                & 6    & 0                         & 81.88          & 58.35          & 54.15          & 47.1          & 626.8                     & 82.03          & 57.65          & 52.8           & 48.0          & 626.9 \\ \cline{1-1}
\multirow{7}{*}{TRADES w/ CTR (ours)} & 6    & 1                         & \textbf{83.52} & \textbf{59.3}  & \textbf{54.2}  & \textbf{47.5} & 655.6                     & 83.23          & 57.1           & 52.45          & 47.7          & 620.9 \\
                               & 6    & 2                         & 84.37          & 59.45          & 53.95          & 47.0          & 631.1                     & 84.15          & 58.2           & 53.05          & 47.7          & 641.3 \\
                               & 6    & 3                         & 84.97          & 58.15          & 52.75          & 47.1          & 651.5                     & \textbf{84.82} & \textbf{58.7}  & \textbf{52.65} & \textbf{48.6} & 626.8 \\
                               & 6    & 4                         & 86.25          & 58.95          & 52.65          & 46.2          & 638.8                     & 85.45          & 58.05          & 52.85          & 46.4          & 622.2 \\
                               & 6    & 5                         & 86.47          & 57.95          & 51.75          & 45.8          & 644.2                     & 85.58          & 59.1           & 52.25          & 46.5          & 645.6 \\
                               & 7    & 1                         & \textbf{83.15} & \textbf{58.2}  & \textbf{53.45} & \textbf{48.1} & 636.7                     & \textbf{82.80} & \textbf{58.05} & \textbf{53.5}  & \textbf{48.6} & 636.0 \\
                               & 8    & 1                         & 82.95          & 58.75          & 54.3           & 47.8          & 633.9                     & 82.69          & 57.65          & 53.1           & 48.1          & 638.6 \\ \hline
\end{tabular}
\end{center}
\end{table*}

\subsubsection{Robustness Evaluation of MDL}
\label{sec:robustness-evaluation-MDL}
\textbf{Comparison among MDL, Dropout and MSD.}
Fig.~\ref{fig:dropout-natural-lines} evaluated the robustness of the DNN with different regularization methods on 4 datasets. Under FGSM and PGD attacks with any strength, the accuracy of MDL is significantly greater than that of dropout and MSD. Therefore, MDL significantly improves the robustness of the DNN in natural training. In addition, under FGSM ($\varepsilon$=8/255) attack, the accuracies of VGG16 and ResNet50 with MDL on CIFAR10 are 61.92\%, 62.84\%, respectively, which have reached a comparable level with that of adversarial training methods.

\textbf{Comparison among More Regularization Methods and More Attack Methods on CIFAR10.}
Fig.~\ref{fig:more_attacks_natural_lines} shows that the robustness of the other regularization methods is basically the same. Under all white-box attacks (except AutoAttack) and all black-box attacks with any attack strength, the accuracy of MDL is significantly greater than that of the other regularization methods. Therefore, The existing regularization methods do not improve the robustness of the DNN, but MDL can significantly improve the robustness of the DNN in natural training on most of white-box attacks and black-box attacks.

In addition, Table~\ref{tab:average-accuracy-for-three-times} shows that MDL slightly increases the build time of the DNN. However, in comparison with the improvement of generalization and robustness, the increased build time of MDL is acceptable.

\textbf{Compatibility between MDL and Dropconnect.}
To verify that MDL is compatible with dropconnect, the dropout operator in MDL replaces with dropconnect. Fig.~\ref{fig:dropconnect-natural-lines} evaluated the robustness of the DNN trained by the combination of MDL and dropconnect on 4 datasets. Under FGSM and PGD attacks with any strength, the accuracy of the combination of MDL and dropconnect is significantly greater than that of only dropconnect. Therefore, the combination of MDL and dropconnect does not affect the significant improvement in the robustness of the DNN.

\subsubsection{Statistical Analysis of MDL}
\label{sec:statistical-analysis-MDL}
\textbf{The Statistical Verification of the Theoretical CT.} According to Section~\ref{sec:generalization-evaluation-MDL}, MDL makes more test inputs to be classified correctly. Meanwhile, as shown in Table~\ref{tab:statistical-analysis-of-CT}, MDL makes the probability of more wrong categories less than the theoretical CT (i.e. $\frac{1}{C-1}$). Therefore, the theoretical analysis of Section~\ref{sec:theoretical-analysis-MDL} is correct.

\textbf{The Exploration of Tighter Empirical CT.} To explore a CT (i.e., the tighter empirical CT) of MDL may be less than the theoretical CT, as shown in Fig.~\ref{fig:tighter-empirical-CT}, with the decreasing of the empirical CT from the theoretical CT, i) the number of wrong categories with probability less than the empirical CT decreases slowly; ii) the gap between MDL and the other methods is gradually widened in the number of wrong categories with probability less than the empirical CT. Therefore, a tighter empirical CT exists, and the effectiveness of MDL is empirically verified in reducing the CT.

\begin{table*}[t]
\begin{center}
\caption{The accuracies (\%) of Fast-AT~\cite{Fast-AT}, Free-AT~\cite{Free-AT} and Madry-AT~\cite{PGD} with or without CTR on each attacks. Time denotes build time (mins). Bold indicates higher performance. Note that due to the overfit of the model, Fast-AT without CTR runs 40 epochs on Multistep.}
\label{tab:accuracies-of-Fast-Free-Madry}
\begin{tabular}{ccccccccccccc}
\hline
                                  &      & \multicolumn{1}{c|}{}       & \multicolumn{5}{c|}{Multistep}                                                               & \multicolumn{5}{c}{Cyclic}                                               \\ \hline
Methods                           & $\gamma$ & \multicolumn{1}{c|}{Epochs} & Natural        & FGSM           & PGD            & AA            & \multicolumn{1}{c|}{Time} & Natural        & FGSM           & PGD            & AA            & Time  \\ \hline
Fast-AT w/o CTR~\cite{Fast-AT}                  & 0    & 40/50                       & 83.82          & 53.95          & 47.1           & 40.8          & 45.7                      & 84.46          & 54.0           & 46.05          & 40.2          & 52.9  \\ \cline{1-1}
\multirow{7}{*}{Fast-AT w/ CTR (ours)}   & 1    & 50                          & 84.78          & 55.65          & 48.6           & 41.2          & 53.3                      & 84.89          & 54.6           & 46.45          & 40.5          & 54.0  \\
                                  & 2    & 50                          & \textbf{84.34} & \textbf{55.55} & \textbf{48.65} & \textbf{42.4} & 57.6                      & \textbf{84.34} & \textbf{55.45} & \textbf{48.25} & \textbf{42.3} & 53.9  \\
                                  & 3    & 50                          & 83.31          & 58.1           & 50.4           & 41.0          & 53.3                      & 84.17          & 57.1           & 49.95          & 41.8          & 53.8  \\
                                  & 4    & 50                          & \textbf{81.26} & \textbf{57.5}  & \textbf{53.7}  & \textbf{42.5} & 53.1                      & \textbf{81.95} & \textbf{57.55} & \textbf{53.45} & \textbf{42.5} & 54.0  \\
                                  & 5    & 50                          & 77.93          & 56.15          & 52.05          & 41.3          & 53.1                      & 77.21          & 55.0           & 52.05          & 41.9          & 53.8  \\
                                  & 4    & -                           & 83.37          & 58.9           & 53.05          & 41.1          & 74.3                      & 83.68          & 58.2           & 52.45          & 41.4          & 75.6  \\
                                  & 5    & -                           & 81.05          & 57.05          & 51.4           & 40.4          & 84.9                      & 81.58          & 57.55          & 53.4           & 40.9          & 96.8  \\ \hline
Free-AT w/o CTR~\cite{Free-AT}                  & 0    & 50                          & 83.3           & 53.95          & 49.65          & 43.0          & 179.7                     & 83.36          & 55.85          & 50.55          & 45.5          & 183.8 \\ \cline{1-1}
\multirow{8}{*}{Free-AT w/ CTR (ours)}   & 1    & 50                          & 83.14          & 55.25          & 50.2           & 43.4          & 182.2                     & \textbf{83.39} & \textbf{56.5}  & \textbf{51.95} & \textbf{46.3} & 179.8 \\
                                  & 2    & 50                          & 82.33          & 55.8           & 51.4           & 44.1          & 179.9                     & 83.31          & 57.1           & 52.7           & 45.2          & 178.8 \\
                                  & 3    & 50                          & 81.25          & 56.25          & 53.1           & 42.8          & 181.3                     & 82.81          & 58.7           & 54.75          & 45.6          & 182.5 \\
                                  & 4    & 50                          & 77.72          & 54.2           & 51.5           & 41.3          & 178.9                     & 81.04          & 57.7           & 54.85          & 44.3          & 182.7 \\
                                  & 1    & -                           & 83.29          & 55.95          & 50.7           & 44.4          & 218.6                     & \textbf{84.15} & \textbf{56.55} & \textbf{51.85} & \textbf{46.2} & 251.7 \\
                                  & 2    & -                           & \textbf{83.41} & \textbf{56.95} & \textbf{52.05} & \textbf{45.1} & 251.8                     & \textbf{83.84} & \textbf{58.05} & \textbf{53.05} & \textbf{46.5} & 250.3 \\
                                  & 3    & -                           & \textbf{82.43} & \textbf{58.25} & \textbf{54.0}  & \textbf{45.2} & 253.8                     & 83.69          & 58.8           & 54.8           & 45.7          & 255.5 \\
                                  & 4    & -                           & 79.89          & 56.0           & 52.0           & 42.7          & 322.0                     & 82.31          & 59.65          & 56.3           & 45.0          & 328.8 \\ \hline
Madry-AT w/o CTR~\cite{PGD}                 & 0    & 50                          & 82.29          & 56.1           & 50.5           & 47.0          & 179.3                     & 82.93          & 55.1           & 50.2           & 47.0          & 178.9 \\ \cline{1-1}
\multirow{10}{*}{Madry-AT w/ CTR (ours)} & 1    & 50                          & 82.2           & 57.4           & 52.45          & 46.6          & 174.4                     & 82.79          & 56.4           & 51.35          & 46.1          & 180.3 \\
                                  & 2    & 50                          & 81.88          & 58.55          & 54.35          & 46.5          & 178.8                     & 82.41          & 57.5           & 52.95          & 46.7          & 175.9 \\
                                  & 3    & 50                          & 80.69          & 59.2           & 56.33          & 47.1          & 179.5                     & \textbf{81.47} & \textbf{59.85} & \textbf{55.55} & \textbf{47.2} & 174.2 \\
                                  & 4    & 50                          & 77.87          & 58.3           & 55.75          & 46.8          & 176.1                     & 78.58          & 57.85          & 55.15          & 45.2          & 179.1 \\
                                  & 5    & 50                          & 73.74          & 55.1           & 52.45          & 43.1          & 174.0                     & 72.5           & 54.35          & 51.9           & 42.1          & 180.0 \\
                                  & 1    & -                           & 82.76          & 57.15          & 53.0           & 46.8          & 191.8                     & 83.9           & 56.85          & 51.15          & 45.3          & 252.4 \\
                                  & 2    & -                           & \textbf{82.69} & \textbf{58.05} & \textbf{53.85} & \textbf{47.5} & 196.6                     & 83.4           & 57.35          & 51.35          & 45.1          & 246.2 \\
                                  & 3    & -                           & 82.72          & 60.3           & 56.7           & 45.9          & 251.3                     & 83.39          & 58.55          & 54.15          & 45.5          & 243.8 \\
                                  & 4    & -                           & 81.32          & 59.05          & 55.3           & 44.7          & 281.7                     & 81.94          & 60.7           & 56.5           & 45.8          & 322.3 \\
                                  & 5    & -                           & 78.01          & 57.8           & 55.25          & 43.7          & 278.4                     & 78.7           & 58.75          & 55.5           & 45.4          & 324   \\ \hline
\end{tabular}
\end{center}
\end{table*}

\subsection{Adversarial Training with CTR}
\label{sec:adversarial-training-with-CTR}
In this section, the robustness evaluation of Fast-AT, Free-AT, Madry-AT and TRADES with and without CTR are compared where CTR is implemented by integrating the STD loss into the loss funtion of original adversarial training. Note that the setup not mentioned in this section is the same as Section~\ref{sec:setup1}.

\subsubsection{Setup}
\label{sec:setup2}
The effectiveness of the proposed standard deviation loss function is evaluated with CIFAR10~\cite{CIFAR} dataset and ResNet18~\cite{ResNet} for adversarial training. Fast-AT~\cite{Fast-AT}, Free-AT~\cite{Free-AT}, Madry-AT~\cite{PGD} and TRADES~\cite{TRADES} are selected as the baselines to evaluate the robustness with and without CTR where the purpose of Fast-AT, Free-AT and Madry-AT methods is to improve  the only robustness of the DNN while the purpose of TRADES is to achieve the trade-off between the robust and natural accuracies.

Because a large number of epochs leads to the overfitting of Fast-AT~\cite{Fast-AT} without CTR (i.e., losing the robustness), Fast-AT with and without CTR train 50 and 40 epochs, respectively. In Madry-AT~\cite{PGD}, the model is basically convergent after 50 epochs of training, and thus Madry-AT with or without CTR trains 50 epochs. Because Free-AT~\cite{Free-AT} is an accelerated version of Madry-AT, the model is convergent after 50 epochs of training with Free-AT, and thus Free-AT with or without CTR trains 50 epochs. The number of the minibatch replays of Free-AT with or without CTR is set as 8. The TRADES with or without CTR trains 200 epochs, which is the same as the original paper of TRADES~\cite{TRADES}.

In addition, for Multistep, due to the initial training difficulties of the SCE and SKL losses, the gradual warmup~\cite{Warmup} is applied, i.e., the learning rate is multiplied by the warmup factor $\kappa_i$:
\begin{align}\label{eq:warmup-factor}
\left\{ \begin{array}{c}
\kappa _{i+1}=\kappa _i\ast \left( 1-\frac{i}{I} \right) +\frac{i}{I}, 1\leq i\leq I \\
\kappa _1=0.001\\
\end{array} \right.
\end{align}
where $I$ is $\frac{1}{10}$ of the total epochs, $\kappa_i$ is the $i$-th epoch of the warmup factor. For Multistep, the learning rate is initially set to 0.1 on Free-AT and TRADES, 0.2 on Fast-AT and Madry-AT, which decay at the $\frac{2}{3}$ and $\frac{5}{6}$ of the total epochs by a factor of 0.1. For Cyclic, the max and min learning rates are respectively set to 0.2 and 0 on Fast-AT, Madry-AT and TRADES, and 0.1 and 0 on Free-AT. In addition to the attacks in Section~\ref{sec:setup1}, momentum iterative FGSM (MIFGSM)~\cite{MIFGSM} is used to evaluate the robustness of the model. The attack strength of all attacks is $\epsilon=\frac{8}{255}$.

\subsubsection{Robustness Evaluation of TRADES w/ or w/o CTR}
\label{sec:robustness-evaluation-of-TRADES}
\textbf{Comparison of TRADES with and without CTR.} As shown in Table~\ref{tab:accuracies-of-TRADES}, in comparison with TRADES without CTR ($\beta$=1)~\cite{TRADES}, TRADES with CTR ($\beta\geq$1 and bold in Table~\ref{tab:accuracies-of-TRADES}) can improve the FGSM accuracy by 2.7\% to 5.15\% on Multistep and 2.45\% to 3.7\% on Cyclic, the PGD accuracy by 2.6\% to 6\% on Multistep and 2.55\% to 5.95\% on Cyclic, the AutoAttack accuracy by 1.7\% to 4.9\% on Multistep and 2.0\% to 4.0\% on Cyclic while keeping the natural accuracy comparable. In comparison with TRADES without CTR ($\beta$=6)~\cite{TRADES}, TRADES with CTR ($\beta\geq$6 and bold in Table~\ref{tab:accuracies-of-TRADES}) can improve the natural accuracy by 1.27\% to 1.64\% on Multistep and 0.77\% to 2.79\% on Cyclic, the AutoAttack accuracy by 0.4\% to 1.0\% on Multistep and 0.6\% on Cyclic while keeping the FGSM accuracy and PGD accuracy almost unchanged. The evaluation of TRADES with and without CTR on more attacks is shown in Appendix A (available in the online supplemental material). Therefore, TRADES with CTR can further improve the robustness of the model.

\textbf{Sensitivity of Hyperparameter $\beta$.} As the hyperparameter $\beta$ increases from 1 to 6, Table~\ref{tab:accuracies-of-TRADES} shows that the natural accuracy of TRADES without CTR decreases faster than that of TRADES with CTR ($\gamma$=3), i.e. 4.79\%$>$2.92\% on Multistep and 4.2\%$>$2.66\% on Cyclic. Therefore, in comparison with TRADES without CTR, TRADES with CTR can increase the hyperparameter $\beta$ to improve the robust accuracy while keeping the natural accuracy almost unchanged.

\textbf{Sensitivity of Hyperparameter $\gamma$.} As the hyperparameter $\gamma$ increases from 1 to 5, for TRADES with CTR, Table~\ref{tab:accuracies-of-TRADES} shows that the natural accuracy increases, and the robust accuracy increases first and then decreases. Therefore, choosing an intermediate value for the hyperparameter $\gamma$ will make the comprehensive performance optimal, such as, $\gamma=$3 for $\beta=$1 and $\gamma=$1 for $\beta=$6 on Multistep.

\textbf{Comparison of Multistep and Cyclic.} The improvement of the robust accuracy on Cyclic is higher than that on Multistep. The improvement of the natural accuracy on Multistep is higher than that on Cyclic. 

In addition, the build time of TRADES with CTR is slightly higher than that without CTR, which is acceptable while being compared with the improvement of robustness.

\subsubsection{Robustness Evaluation of Fast-AT, Free-AT and Madry-AT w/ or w/o CTR}
\label{sec:robustness-evaluation-of-Fast-Free-Madry}
Note that due to the slow convergence of adversarial training with CTR, we further explore the robustness of the model by appropriately increasing the training time.

\textbf{Comparison of Fast-AT, Free-AT and Madry-AT with and without CTR.} As shown in Table~\ref{tab:accuracies-of-Fast-Free-Madry}, in comparison with Fast-AT without CTR~\cite{Fast-AT}, Fast-AT with CTR (bold in Table~\ref{tab:accuracies-of-Fast-Free-Madry}) can improve the FGSM accuracy by 1.6\% to 3.55\% on Multistep and 1.45\% to 3.55\% on Cyclic, the PGD accuracy by 1.55\% to 6.6\% on Multistep and 2.2\% to 7.4\% on Cyclic, the AutoAttack accuracy by 1.6\% to 1.7\% on Multistep and 2.1\% to 2.3\% on Cyclic while achieving the comparable or slight decreasing natural accuracy.

In comparison with Free-AT without CTR~\cite{Free-AT}, Free-AT with CTR (bold in Table~\ref{tab:accuracies-of-Fast-Free-Madry}) can improve the FGSM accuracy by 3.0\% to 4.3\% on Multistep and 0.7\% to 2.2\% on Cyclic, the PGD accuracy by 2.4\% to 4.35\% on Multistep and 1.3\% to 2.5\% on Cyclic, the AutoAttack accuracy by 2.1\% to 2.2\% on Multistep and 0.7\% to 1.0\% on Cyclic while keeping the comparable natural accuracy.

In comparison with Madry-AT without CTR~\cite{PGD}, Madry-AT with CTR (bold in Table~\ref{tab:accuracies-of-Fast-Free-Madry}) can improve the FGSM accuracy by 1.95\% on Multistep and 4.75\% on Cyclic, the PGD accuracy by 3.35\% on Multistep and 5.35\% on Cyclic while keeping the comparable natural and AutoAttack accuracies.

The evaluation of Fast-AT, Free-AT and Madry-AT with or without CTR on more attacks is shown in Appendix B (available in the online supplemental material).

Therefore, the robustness of the model with CTR is higher than that without CTR on Fast-AT~\cite{Fast-AT}, Free-AT~\cite{Free-AT} and Madry-AT~\cite{PGD} against white-box gradient attacks.

\textbf{Sensitivity of Hyperparameter $\gamma$.} Using the same training time, as the hyperparameter $\gamma$ increases from 1 to 5, Table~\ref{tab:accuracies-of-Fast-Free-Madry} shows that the natural accuracy decreases, and the robust accuracy increases first and then decreases. Therefore, choosing an intermediate value for the hyperparameter $\gamma$ will make the comprehensive performance optimal, such as, $\gamma=$2 for Fast-AT, Free-AT and Madry-AT on Multistep.

\textbf{Comparison of Multistep and Cyclic.} Cyclic is better than Multistep on the natural and robust accuracies.

In addition, Table~\ref{tab:accuracies-of-Fast-Free-Madry} shows that the average training time of each epoch is almost the same on Fast-AT, Free-AT and Madry-AT with and without CTR. However, due to the slow convergence of adversarial training with CTR, the slight increased training time is acceptable for the improvement of the robustness.

\begin{table}[t]
\begin{center}
\caption{The attack success rates (\%) of the different loss functions based gradient attacks on naturally trained and adversarially trained models. Note that NT denotes natural training and TR denotes TRADES.}
\label{tab:attacks-comparison}
\begin{tabular}{ccccccc}
\hline
Methods               & Loss & NT        & \begin{tabular}[c]{@{}c@{}}Fast\\ -AT\end{tabular} & \begin{tabular}[c]{@{}c@{}}Free\\ -AT\end{tabular} & \begin{tabular}[c]{@{}c@{}}Madry\\ -AT\end{tabular} & TR         \\ \hline
\multirow{5}{*}{FGSM} & CE~\cite{FGSM}   & 47.45          & 46.05                                              & 46.05                                              & 43.9                                                & 45.5           \\
                      & SCE (ours)  & 47.55          & \textbf{46.35}                                     & 46.3                                               & \textbf{44.65}                                      & 45.7           \\ \cline{2-2}
                      & KL~\cite{TRADES}   & 26.45          & 24.35                                              & 25.0                                               & 25.05                                               & 22.2           \\
                      & SKL (ours)  & 24.9           & 25.75                                              & 25.55                                              & 24.75                                               & 22.05          \\ \cline{2-2}
                      & STD (ours)  & \textbf{53.75} & 45.95                                              & \textbf{46.95}                                     & 44.55                                               & \textbf{46.6}  \\ \hline
\multirow{5}{*}{PGD}  & CE~\cite{FGSM}   & 75.7           & 53.0                                               & 50.45                                              & 48.65                                               & 53.4           \\
                      & SCE (ours)  & 76.05          & \textbf{54.3}                                      & \textbf{51.75}                                     & \textbf{51.45}                                      & \textbf{54.25} \\ \cline{2-2}
                      & KL~\cite{TRADES}   & 43.55          & 39.35                                              & 36.15                                              & 36.6                                                & 37.85          \\
                      & SKL (ours)  & 43.35          & 40.05                                              & 37.05                                              & 36.65                                               & 38.95          \\ \cline{2-2}
                      & STD (ours)  & \textbf{82.25} & 52.9                                               & 51.2                                               & 49.55                                               & \textbf{54.25} \\ \hline
\multirow{5}{*}{APGD} & CE~\cite{FGSM}   & 89.45          & 55.8                                               & 52.55                                              & 50.7                                                & 56.3           \\
                      & SCE (ours)  & 89.6           & \textbf{57.4}                                      & \textbf{54.85}                                     & \textbf{53.0}                                       & \textbf{57.2}  \\ \cline{2-2}
                      & KL~\cite{TRADES}   & 59.7           & 53.8                                               & 39.2                                               & 39.25                                               & 43.3           \\
                      & SKL (ours)  & 59.75          & 54.25                                              & 39.6                                               & 39.4                                                & 43.55          \\ \cline{2-2}
                      & STD (ours)  & \textbf{90.4}  & 55.8                                               & 53.0                                               & 51.4                                                & 56.65          \\ \hline
\end{tabular}
\end{center}
\end{table}
\begin{figure}[t]
\centering
\includegraphics[width=0.99\columnwidth]{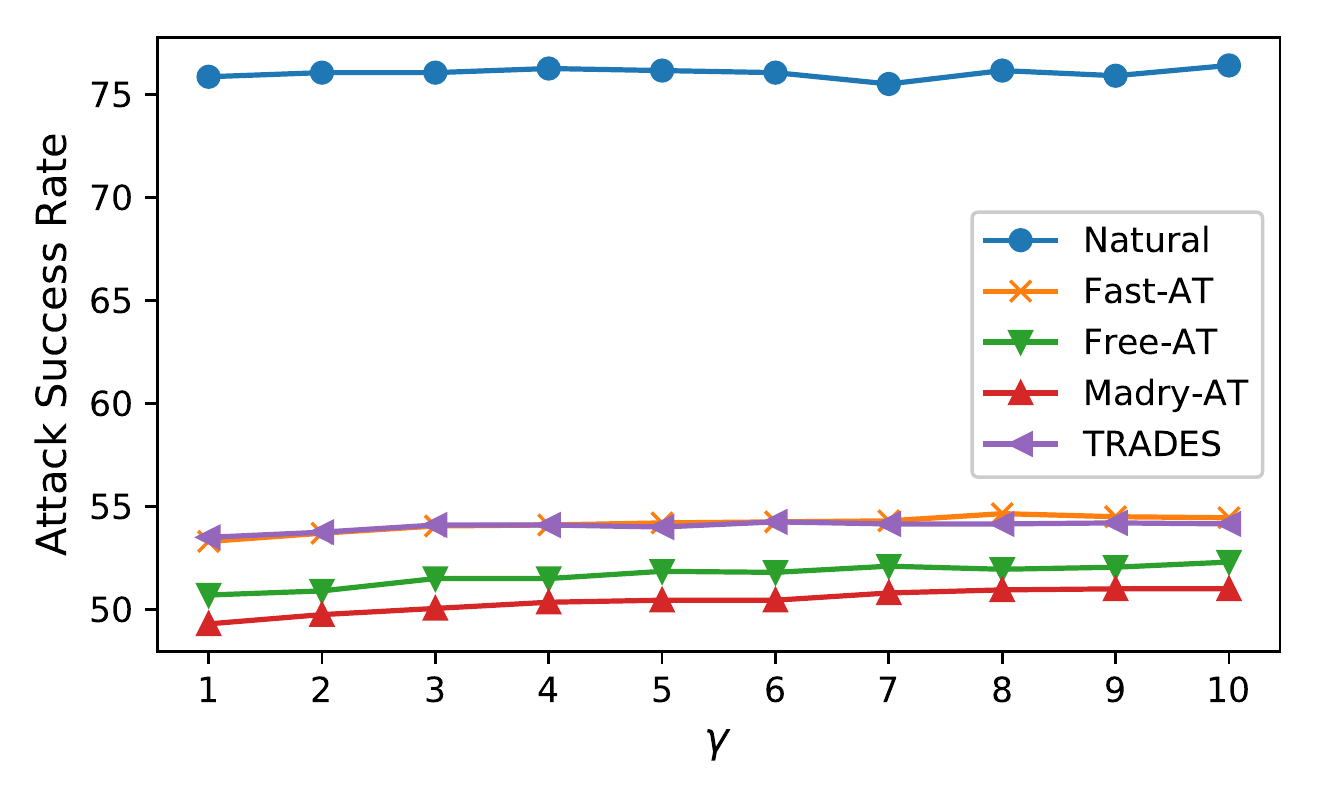} 
\caption{The attack success rate (\%) lines of the SCE loss based PGD on naturally trained and adversarially trained models with different hyperparameter $\gamma$.}
\label{fig:pgd_sce_with_different_gama}
\end{figure}

\subsection{The STD Loss Based Gradient Attack Methods}
\label{sec:std-loss-based-attack-methods}
In this section, the attack success rates (ASR) of the gradient attacks based on different loss functions are compared. Note that the setup not mentioned in this section is the same as Sections~\ref{sec:setup1} and \ref{sec:setup2}.

\subsubsection{Setup}
\label{sec:setup3}
The effectiveness  of the STD loss based gradient attacks is evaluated with CIFAR10~\cite{CIFAR} dataset and ResNet18~\cite{ResNet}. FGSM~\cite{FGSM}, PGD~\cite{PGD} and APGD~\cite{AutoAttack} are selected as the baselines to attack the models trained by natural training (NT), Fast-AT~\cite{Fast-AT}, Free-AT~\cite{Free-AT}, Madry-AT~\cite{PGD} and TRADES~\cite{TRADES}, respectively. The attack strength is $\epsilon=\frac{2}{255}$ on the naturally trained model and $\epsilon=\frac{8}{255}$ on the adversarially trained model. The hyperparameter $\beta$ is set as 5 for the SCE loss and 1 for the SKL loss.

\subsubsection{Evaluation of the STD Loss and Its Variants Based Gradient Attack Methods}
\label{sec:evaluation-STD-loss-attack}
\textbf{ASR Comparison on Naturally Trained Model.} As shown in Table~\ref{tab:attacks-comparison}, the ASR of the STD loss based gradient attacks on the naturally trained model is significantly greater than the other losses based gradient attacks, which verifies the correctness of the theoretical analysis in Section~\ref{sec:STD-based-adversarial-attacks}.

\textbf{ASR Comparison on Adversarially Trained Model.} As shown in Table~\ref{tab:attacks-comparison}, the ASR of the STD loss based gradient attacks on the adversarially trained model is greater than the CE loss based gradient attacks, which verifies the correctness of Proposition~\ref{prop:STD-better-than-CE}.

In most cases, the ASR of the SCE loss based gradient attacks on the adversarially trained model is greater than that of the STD loss based gradient attacks. Because the output of the correct category has strong robustness in the adversarially trained model, the output of the correct category on the STD loss decreases passively slowly (Note that in the softmax function, the input of one category becomes larger, which leads to the output probability of other categories becoming smaller). However, on the SCE loss, the output of the correct category decreases actively faster when the descending gradient of the correct category is 0 in the reverse derivation of the STD loss and greater than 0 in that of the SCE loss, which can be obtained from the proof of Propositions~\ref{prop:MadryAT-with-CTR-better-than-without} and \ref{prop:STD-better-than-CE}. Therefore, i) on the adversarially trained model, the ASR of the SCE loss based gradient attacks is greater than that of the STD loss based; ii) in the theoretical analysis of Section~\ref{sec:STD-based-adversarial-attacks}, to verify that the STD loss based gradient attacks are better than the SCE loss based, an assumption that the output of the correct category has no strong robustness is required.

In addition, the ASR of the SCE (or SKL) loss based gradient attacks is greater than that of the CE (or KL) loss, which verifies that the adversarial examples generated by the SCE or SKL loss for adversarial training are more effective than that by the CE or KL loss in Propositions~\ref{prop:MadryAT-with-CTR-better-than-without} and \ref{prop:TRADES-with-CTR-better-than-without} of Section~\ref{sec:adversarial-training-with-CTR-theoretical-analysis}.

Overall, the STD loss is the key role to achieve high ASR on the naturally trained model and adversarially trained model.

\textbf{Sensitivity of Hyperparameter $\gamma$.} As shown in Fig.~\ref{fig:pgd_sce_with_different_gama}, with the increase of $\gamma$, the ASR of the SCE based PGD increases slowly and linearly, which demonstrates that increasing the weight of the STD function in the SCE loss can improve the ASR.

\section{Related Work}
\label{sec:related-work}
In this section, related regularization methods and adversarial training methods are reviewed.

\subsection{Regularization Methods}
Zhang et al.~\cite{Rethinking-Generalization} concluded two types of regularization methods: explicit regularization and implicit regularization, but the improvement of the DNNs is limited. Novak et al.~\cite{Sensitivity-and-Generalization} explored the sensitivity of trained neural networks in perturbation, and the experimental results show that poor generalization corresponds to lower robustness and good generalization gives rise to more robust functions. Presently, the regularization methods can be categorized into three types: dropout, label smoothing and, batch normalization.

\subsubsection{Dropout and Its Variants.}
StochasticDepth~\cite{Stochastic-Depth} was proposed to avoid overfitting and improve the inference speed of ResNet by adding a probability on residual connect. Different from dropout which randomly selects the neurons to set to zero, weighted channel dropout~\cite{Weighted-Channel-Dropout} operates on the channels in the stack of convolutional layers. Some Dropout variants are suitable for different DNN structures and different applications, such as auto-encoder~\cite{Meta-Dropout}, Transformer~\cite{Structured-Dropout}, recurrent neural network~\cite{Message-Dropout}, image restoration~\cite{Self2self-with-Dropout}, adversarial training~\cite{Adversarial-Training-with-Dropout} and reinforcement learning~\cite{Message-Dropout}.

For dynamic dropout variants, their parameters change dynamically based on prior information. A reinforcement learning approach~\cite{Auto-Dropout} was proposed to find better dropout patterns for various architectures. Contextual dropout~\cite{Contextual-Dropout} was proposed to determine the dropout rate by input covariates. In group-wise dynamic dropout~\cite{Group-Wise-Dynamic-Dropout}, the dropout rate of different groups is dynamically changed based on local feature densities. Guided dropout~\cite{Guided-Dropout} chooses guided neurons for intelligent dropout, which leads to better generalization as compared to the traditional dropout.

Arora et al.~\cite{Dropout-Capacity-Control} investigated the capacity control of dropout in various machine learning problems, and calculated the generalization error bound by Rademacher complexity. Zhang et al.~\cite{Interpreting-and-Boosting-Dropout} verified that dropout can prevent the co-adaptation of features in the DNN with theories and experiments. Wei et al.~\cite{Implicit-and-Explicit-Regularization} analyzed the effectiveness of dropout from two aspects of explicit regularization effect and implicit regularization effect, where explicit regularization effect is explained as that dropout changes the objective loss function of neural networks and implicit regularization effect is regarded as that dropout introduces mean-zero noise in gradient. Pal et al.~\cite{Regularization-Properties-of-Dropblock} theoretically verified that dropblock induces spectral k-support norm regularization and dropconnect is equivalent to dropout for single hidden-layer linear networks.

\subsubsection{Label Smoothing and Its Variants.}
Bahri et al.~\cite{Locally-Adaptive-Label-Smoothing} presented several baselines to reduce churn and a K-Nearest Neighbor based label smoothing outperforming the baselines on churn. Ghoshal et al.~\cite{Low-Rank-Adaptive-Label-Smoothing} presented a low-rank adaptive label smoothing, which can summarize label smoothing and be adaptive to the latent structure of label space in structured representations. Rosenfeld et al.~\cite{Randomized-Smoothing} presented a unifying view of randomized smoothing over arbitrary functions and a new strategy for building classifiers that are robust to general adversarial attacks. Some label smoothing variants appear in co-saliency detection~\cite{Co-Saliency-Detection-Label-Smoothing}, recommender systems~\cite{LS-for-Recommender-Systems} and arbitrary-oriented object detection~\cite{Circular-Smooth-Label}.

Label smoothing is not the root cause of the poor performance of knowledge distillation. Lukasik et al.~\cite{LS-Mitigate-Label-Noise} studied how label smoothing relates to loss-correction techniques, and show that when distilling models from noisy data, label smoothing of the teacher is beneficial, which is in contrast to recent studies for noise-free problems. Shen et al.~\cite{LS-Incompatible-Knowledge-Distillation} explained that teachers with label smoothing provide less subtle semantic information than without label soothing, which is not the main reason for label smoothing failure but the training dataset appears a long-tailed distribution, and the number of the class is increased.

\subsubsection{Batch Normalization and Its Variants.}
Batch normalization~\cite{Batch-Normalization} solved the problem of internal covariate shift, which leads to an unstable training process and slow convergence speed. Summers et al.~\cite{Improve-Batch-Normalization} presented four tricks to improve batch normalization, especially proposing a method for reasoning about the current example in inference normalization statistics, fixing a training vs. inference discrepancy. Transformer and meta learning also have their batch normalization versions, PowerNorm~\cite{Power-Norm} and TaskNorm~\cite{Task-Norm} respectively. Batch normalization has been demonstrated as an effective strategy to avoid networks collapsing quickly with depth~\cite{BN-Avoids-Ranks-Collapse}. Li et al.~\cite{Disharmony-Dropout-BN} proposed two solutions to alleviate the variance shift between dropout and batch normalization.

\subsection{Adversarial Training Methods}
Due to the linear nature~\cite{FGSM} of the DNN, the DNN is vulnerable to the imperceptible perturbation~\cite{FGSM,PGD}. The adversarial training is the most effective defense methods, which are divided into two categories based on the purposes: the only adversarial samples trained~\cite{Fast-AT,Free-AT} and the trade-off between robustness and accuracy~\cite{TRADES,ATTA,FAT}, which have been introduced in Section~\ref{sec:introduction}.

\section{Conclusion}
\label{sec:conclusion}
In this paper, we propose and analyze a confidence threshold reduction (CTR) method to improve the generalization and robustness of the model simultaneously. Specifically, for natural training, to reduce the confidence threshold (CT), we propose a mask-guided divergence loss function (MDL), which consists of a cross-entropy loss term and an orthogonal term. The target of the orthogonal term is to increase the diversity of the ensemble sub-DNN and reduce the CT. The empirical and theoretical analysis demonstrates that the MDL loss improves the generalization and robustness simultaneously with the CTR. However, the model robustness improvement of natural training with the CTR is not comparable to that of adversarial training. Therefore, for adversarial training, to reduce the CT, we propose a standard deviation loss (STD) to further improve the robustness of the model by integrating it into the loss function of adversarial training. The empirical and theoretical analysis demonstrates that the STD loss based loss function can further improve the robustness of the adversarially trained model while keeping the natural accuracy unchanged or slightly improved. In addition,  the STD loss and its variants based gradient attack methods can achieve great attack success rate.


%





\ifCLASSOPTIONcaptionsoff
  \newpage
\fi



%
\bibliographystyle{IEEEtran}
\bibliography{ref}

\begin{thebibliography}{10}
\providecommand{\url}[1]{#1}
\csname url@samestyle\endcsname
\providecommand{\newblock}{\relax}
\providecommand{\bibinfo}[2]{#2}
\providecommand{\BIBentrySTDinterwordspacing}{\spaceskip=0pt\relax}
\providecommand{\BIBentryALTinterwordstretchfactor}{4}
\providecommand{\BIBentryALTinterwordspacing}{\spaceskip=\fontdimen2\font plus
\BIBentryALTinterwordstretchfactor\fontdimen3\font minus
  \fontdimen4\font\relax}
\providecommand{\BIBforeignlanguage}[2]{{%
\expandafter\ifx\csname l@#1\endcsname\relax
\typeout{** WARNING: IEEEtran.bst: No hyphenation pattern has been}%
\typeout{** loaded for the language `#1'. Using the pattern for}%
\typeout{** the default language instead.}%
\else
\language=\csname l@#1\endcsname
\fi
#2}}
\providecommand{\BIBdecl}{\relax}
\BIBdecl

\bibitem{FGSM}
I.~J. Goodfellow, J.~Shlens, and C.~Szegedy, ``Explaining and harnessing
  adversarial examples,'' in \emph{{ICLR} (Poster)}, 2015.

\bibitem{PGD}
A.~Madry, A.~Makelov, L.~Schmidt, D.~Tsipras, and A.~Vladu, ``Towards deep
  learning models resistant to adversarial attacks,'' in \emph{{ICLR}
  (Poster)}.\hskip 1em plus 0.5em minus 0.4em\relax OpenReview.net, 2018.

\bibitem{CW}
N.~Carlini and D.~A. Wagner, ``Towards evaluating the robustness of neural
  networks,'' in \emph{{IEEE} Symposium on Security and Privacy}.\hskip 1em
  plus 0.5em minus 0.4em\relax {IEEE} Computer Society, 2017, pp. 39--57.

\bibitem{BIM}
A.~Kurakin, I.~J. Goodfellow, and S.~Bengio, ``Adversarial examples in the
  physical world,'' in \emph{{ICLR} (Workshop)}.\hskip 1em plus 0.5em minus
  0.4em\relax OpenReview.net, 2017.

\bibitem{MIFGSM}
Y.~Dong, F.~Liao, T.~Pang, H.~Su, J.~Zhu, X.~Hu, and J.~Li, ``Boosting
  adversarial attacks with momentum,'' in \emph{{CVPR}}.\hskip 1em plus 0.5em
  minus 0.4em\relax Computer Vision Foundation / {IEEE} Computer Society, 2018,
  pp. 9185--9193.

\bibitem{AutoAttack}
F.~Croce and M.~Hein, ``Reliable evaluation of adversarial robustness with an
  ensemble of diverse parameter-free attacks,'' in \emph{{ICML}}, ser.
  Proceedings of Machine Learning Research, vol. 119.\hskip 1em plus 0.5em
  minus 0.4em\relax {PMLR}, 2020, pp. 2206--2216.

\bibitem{Square}
M.~Andriushchenko, F.~Croce, N.~Flammarion, and M.~Hein, ``Square attack: {A}
  query-efficient black-box adversarial attack via random search,'' in
  \emph{{ECCV} {(23)}}, ser. Lecture Notes in Computer Science, vol.
  12368.\hskip 1em plus 0.5em minus 0.4em\relax Springer, 2020, pp. 484--501.

\bibitem{AoA}
S.~Chen, Z.~He, C.~Sun, J.~Yang, and X.~Huang, ``Universal adversarial attack
  on attention and the resulting dataset damagenet,'' \emph{{IEEE} Trans.
  Pattern Anal. Mach. Intell.}, vol.~44, no.~4, pp. 2188--2197, 2022.

\bibitem{UAP}
K.~R. Mopuri, A.~Ganeshan, and R.~V. Babu, ``Generalizable data-free objective
  for crafting universal adversarial perturbations,'' \emph{{IEEE} Trans.
  Pattern Anal. Mach. Intell.}, vol.~41, no.~10, pp. 2452--2465, 2019.

\bibitem{Dropout}
G.~E. Hinton, N.~Srivastava, A.~Krizhevsky, I.~Sutskever, and R.~Salakhutdinov,
  ``Improving neural networks by preventing co-adaptation of feature
  detectors,'' 2012.

\bibitem{Dropconnect}
L.~Wan, M.~D. Zeiler, S.~Zhang, Y.~LeCun, and R.~Fergus, ``Regularization of
  neural networks using dropconnect,'' in \emph{{ICML} {(3)}}, ser. {JMLR}
  Workshop and Conference Proceedings, vol.~28.\hskip 1em plus 0.5em minus
  0.4em\relax JMLR.org, 2013, pp. 1058--1066.

\bibitem{Dropblock}
G.~Ghiasi, T.~Lin, and Q.~V. Le, ``Dropblock: {A} regularization method for
  convolutional networks,'' in \emph{NeurIPS}, 2018, pp. 10\,750--10\,760.

\bibitem{Multi-Sample-Dropout}
H.~Inoue, ``Multi-sample dropout for accelerated training and better
  generalization,'' 2019.

\bibitem{Beyond-Dropout}
Y.~Tang, Y.~Wang, Y.~Xu, B.~Shi, C.~Xu, C.~Xu, and C.~Xu, ``Beyond dropout:
  Feature map distortion to regularize deep neural networks,'' in
  \emph{{AAAI}}.\hskip 1em plus 0.5em minus 0.4em\relax {AAAI} Press, 2020, pp.
  5964--5971.

\bibitem{Dropout-Prevent-Overfitting}
N.~Srivastava, G.~E. Hinton, A.~Krizhevsky, I.~Sutskever, and R.~Salakhutdinov,
  ``Dropout: a simple way to prevent neural networks from overfitting,''
  \emph{J. Mach. Learn. Res.}, vol.~15, no.~1, pp. 1929--1958, 2014.

\bibitem{Dropout-Data-Augmentation}
K.~R. Konda, X.~Bouthillier, R.~Memisevic, and P.~Vincent, ``Dropout as data
  augmentation,'' 2015.

\bibitem{Free-AT}
A.~Shafahi, M.~Najibi, A.~Ghiasi, Z.~Xu, J.~P. Dickerson, C.~Studer, L.~S.
  Davis, G.~Taylor, and T.~Goldstein, ``Adversarial training for free!'' in
  \emph{NeurIPS}, 2019, pp. 3353--3364.

\bibitem{Fast-AT}
E.~Wong, L.~Rice, and J.~Z. Kolter, ``Fast is better than free: Revisiting
  adversarial training,'' in \emph{{ICLR}}.\hskip 1em plus 0.5em minus
  0.4em\relax OpenReview.net, 2020.

\bibitem{TRADES}
H.~Zhang, Y.~Yu, J.~Jiao, E.~P. Xing, L.~E. Ghaoui, and M.~I. Jordan,
  ``Theoretically principled trade-off between robustness and accuracy,'' in
  \emph{{ICML}}, ser. Proceedings of Machine Learning Research, vol.~97.\hskip
  1em plus 0.5em minus 0.4em\relax {PMLR}, 2019, pp. 7472--7482.

\bibitem{ATTA}
H.~Zheng, Z.~Zhang, J.~Gu, H.~Lee, and A.~Prakash, ``Efficient adversarial
  training with transferable adversarial examples,'' in \emph{{CVPR}}.\hskip
  1em plus 0.5em minus 0.4em\relax Computer Vision Foundation / {IEEE}, 2020,
  pp. 1178--1187.

\bibitem{FAT}
J.~Zhang, X.~Xu, B.~Han, G.~Niu, L.~Cui, M.~Sugiyama, and M.~S. Kankanhalli,
  ``Attacks which do not kill training make adversarial learning stronger,'' in
  \emph{{ICML}}, ser. Proceedings of Machine Learning Research, vol. 119.\hskip
  1em plus 0.5em minus 0.4em\relax {PMLR}, 2020, pp. 11\,278--11\,287.

\bibitem{QIFGSM}
X.~Yang, J.~Lin, H.~Zhang, X.~Yang, and P.~Zhao, ``Enhancing the
  transferability of adversarial examples via a few queries,'' \emph{CoRR},
  vol. abs/2205.09518, 2022.

\bibitem{Auto-Dropout}
H.~Pham and Q.~V. Le, ``Autodropout: Learning dropout patterns to regularize
  deep networks,'' in \emph{{AAAI}}.\hskip 1em plus 0.5em minus 0.4em\relax
  {AAAI} Press, 2021, pp. 9351--9359.

\bibitem{Contextual-Dropout}
X.~Fan, S.~Zhang, K.~Tanwisuth, X.~Qian, and M.~Zhou, ``Contextual dropout: An
  efficient sample-dependent dropout module,'' in \emph{{ICLR}}.\hskip 1em plus
  0.5em minus 0.4em\relax OpenReview.net, 2021.

\bibitem{Jumpout}
S.~Wang, T.~Zhou, and J.~A. Bilmes, ``Jumpout : Improved dropout for deep
  neural networks with relus,'' in \emph{{ICML}}, ser. Proceedings of Machine
  Learning Research, vol.~97.\hskip 1em plus 0.5em minus 0.4em\relax {PMLR},
  2019, pp. 6668--6676.

\bibitem{Structured-Dropout}
A.~Fan, E.~Grave, and A.~Joulin, ``Reducing transformer depth on demand with
  structured dropout,'' in \emph{{ICLR}}.\hskip 1em plus 0.5em minus
  0.4em\relax OpenReview.net, 2020.

\bibitem{Meta-Dropout}
H.~Lee, T.~Nam, E.~Yang, and S.~J. Hwang, ``Meta dropout: Learning to perturb
  latent features for generalization,'' in \emph{{ICLR}}.\hskip 1em plus 0.5em
  minus 0.4em\relax OpenReview.net, 2020.

\bibitem{Weighted-Channel-Dropout}
S.~Hou and Z.~Wang, ``Weighted channel dropout for regularization of deep
  convolutional neural network,'' in \emph{{AAAI}}.\hskip 1em plus 0.5em minus
  0.4em\relax {AAAI} Press, 2019, pp. 8425--8432.

\bibitem{Guided-Dropout}
R.~Keshari, R.~Singh, and M.~Vatsa, ``Guided dropout,'' in \emph{{AAAI}}.\hskip
  1em plus 0.5em minus 0.4em\relax {AAAI} Press, 2019, pp. 4065--4072.

\bibitem{Message-Dropout}
W.~Kim, M.~Cho, and Y.~Sung, ``Message-dropout: An efficient training method
  for multi-agent deep reinforcement learning,'' in \emph{{AAAI}}.\hskip 1em
  plus 0.5em minus 0.4em\relax {AAAI} Press, 2019, pp. 6079--6086.

\bibitem{Adversarial-Dropout}
S.~Park, K.~Song, M.~Ji, W.~Lee, and I.~Moon, ``Adversarial dropout for
  recurrent neural networks,'' in \emph{{AAAI}}.\hskip 1em plus 0.5em minus
  0.4em\relax {AAAI} Press, 2019, pp. 4699--4706.

\bibitem{Warmup}
P.~Goyal, P.~Doll{\'{a}}r, R.~B. Girshick, P.~Noordhuis, L.~Wesolowski,
  A.~Kyrola, A.~Tulloch, Y.~Jia, and K.~He, ``Accurate, large minibatch {SGD:}
  training imagenet in 1 hour,'' \emph{CoRR}, vol. abs/1706.02677, 2017.

\bibitem{Cyclic}
L.~N. Smith, ``Cyclical learning rates for training neural networks,'' in
  \emph{{WACV}}.\hskip 1em plus 0.5em minus 0.4em\relax {IEEE} Computer
  Society, 2017, pp. 464--472.

\bibitem{Focal-Loss}
T.~Lin, P.~Goyal, R.~B. Girshick, K.~He, and P.~Doll{\'{a}}r, ``Focal loss for
  dense object detection,'' pp. 2999--3007, 2017.

\bibitem{Label-Relaxation}
J.~Lienen and E.~H{\"{u}}llermeier, ``From label smoothing to label
  relaxation,'' in \emph{{AAAI}}.\hskip 1em plus 0.5em minus 0.4em\relax {AAAI}
  Press, 2021, pp. 8583--8591.

\bibitem{Label-Smoothing}
R.~M{\"{u}}ller, S.~Kornblith, and G.~E. Hinton, ``When does label smoothing
  help?'' in \emph{NeurIPS}, 2019, pp. 4696--4705.

\bibitem{Online-Label-Smoothing}
C.~Zhang, P.~Jiang, Q.~Hou, Y.~Wei, Q.~Han, Z.~Li, and M.~Cheng, ``Delving deep
  into label smoothing,'' \emph{{IEEE} Trans. Image Process.}, vol.~30, pp.
  5984--5996, 2021.

\bibitem{MNIST}
Y.~LeCun, L.~Bottou, Y.~Bengio, and P.~Haffner, ``Gradient-based learning
  applied to document recognition,'' \emph{Proc. {IEEE}}, vol.~86, no.~11, pp.
  2278--2324, 1998.

\bibitem{Fashion-MNIST}
H.~Xiao, K.~Rasul, and R.~Vollgraf, ``Fashion-mnist: a novel image dataset for
  benchmarking machine learning algorithms,'' 2017.

\bibitem{CIFAR}
A.~Krizhevsky, ``Learning multiple layers of features from tiny images,'' 2009.

\bibitem{MNIST-Net}
N.~Papernot, P.~D. McDaniel, X.~Wu, S.~Jha, and A.~Swami, ``Distillation as a
  defense to adversarial perturbations against deep neural networks,'' in
  \emph{{IEEE} Symposium on Security and Privacy}.\hskip 1em plus 0.5em minus
  0.4em\relax {IEEE} Computer Society, 2016, pp. 582--597.

\bibitem{VGG16}
K.~Simonyan and A.~Zisserman, ``Very deep convolutional networks for
  large-scale image recognition,'' in \emph{{ICLR}}, 2015.

\bibitem{ResNet}
K.~He, X.~Zhang, S.~Ren, and J.~Sun, ``Deep residual learning for image
  recognition,'' in \emph{{CVPR}}.\hskip 1em plus 0.5em minus 0.4em\relax
  {IEEE} Computer Society, 2016, pp. 770--778.

\bibitem{t-SNE}
\BIBentryALTinterwordspacing
L.~van~der Maaten and G.~Hinton, ``Visualizing data using t-sne,''
  \emph{Journal of Machine Learning Research}, vol.~9, no.~86, pp. 2579--2605,
  2008. [Online]. Available:
  \url{http://jmlr.org/papers/v9/vandermaaten08a.html}
\BIBentrySTDinterwordspacing

\bibitem{Rethinking-Generalization}
C.~Zhang, S.~Bengio, M.~Hardt, B.~Recht, and O.~Vinyals, ``Understanding deep
  learning requires rethinking generalization,'' 2017.

\bibitem{Sensitivity-and-Generalization}
R.~Novak, Y.~Bahri, D.~A. Abolafia, J.~Pennington, and J.~Sohl{-}Dickstein,
  ``Sensitivity and generalization in neural networks: an empirical study,'' in
  \emph{{ICLR} (Poster)}.\hskip 1em plus 0.5em minus 0.4em\relax
  OpenReview.net, 2018.

\bibitem{Stochastic-Depth}
G.~Huang, Y.~Sun, Z.~Liu, D.~Sedra, and K.~Q. Weinberger, ``Deep networks with
  stochastic depth,'' in \emph{{ECCV} {(4)}}, ser. Lecture Notes in Computer
  Science, vol. 9908.\hskip 1em plus 0.5em minus 0.4em\relax Springer, 2016,
  pp. 646--661.

\bibitem{Self2self-with-Dropout}
Y.~Quan, M.~Chen, T.~Pang, and H.~Ji, ``Self2self with dropout: Learning
  self-supervised denoising from single image,'' in \emph{{CVPR}}.\hskip 1em
  plus 0.5em minus 0.4em\relax Computer Vision Foundation / {IEEE}, 2020, pp.
  1887--1895.

\bibitem{Adversarial-Training-with-Dropout}
V.~B. S. and R.~V. Babu, ``Single-step adversarial training with dropout
  scheduling,'' in \emph{{CVPR}}.\hskip 1em plus 0.5em minus 0.4em\relax
  Computer Vision Foundation / {IEEE}, 2020, pp. 947--956.

\bibitem{Group-Wise-Dynamic-Dropout}
Z.~Ke, Z.~Wen, W.~Xie, Y.~Wang, and L.~Shen, ``Group-wise dynamic dropout based
  on latent semantic variations,'' in \emph{{AAAI}}.\hskip 1em plus 0.5em minus
  0.4em\relax {AAAI} Press, 2020, pp. 11\,229--11\,236.

\bibitem{Dropout-Capacity-Control}
R.~Arora, P.~Bartlett, P.~Mianjy, and N.~Srebro, ``Dropout: Explicit forms and
  capacity control,'' in \emph{{ICML}}, ser. Proceedings of Machine Learning
  Research, vol. 139.\hskip 1em plus 0.5em minus 0.4em\relax {PMLR}, 2021, pp.
  351--361.

\bibitem{Interpreting-and-Boosting-Dropout}
H.~Zhang, S.~Li, Y.~Ma, M.~Li, Y.~Xie, and Q.~Zhang, ``Interpreting and
  boosting dropout from a game-theoretic view,'' in \emph{{ICLR}}.\hskip 1em
  plus 0.5em minus 0.4em\relax OpenReview.net, 2021.

\bibitem{Implicit-and-Explicit-Regularization}
C.~Wei, S.~M. Kakade, and T.~Ma, ``The implicit and explicit regularization
  effects of dropout,'' in \emph{{ICML}}, ser. Proceedings of Machine Learning
  Research, vol. 119.\hskip 1em plus 0.5em minus 0.4em\relax {PMLR}, 2020, pp.
  10\,181--10\,192.

\bibitem{Regularization-Properties-of-Dropblock}
A.~Pal, C.~Lane, R.~Vidal, and B.~D. Haeffele, ``On the regularization
  properties of structured dropout,'' in \emph{{CVPR}}.\hskip 1em plus 0.5em
  minus 0.4em\relax Computer Vision Foundation / {IEEE}, 2020, pp. 7668--7676.

\bibitem{Locally-Adaptive-Label-Smoothing}
D.~Bahri and H.~Jiang, ``Locally adaptive label smoothing improves predictive
  churn,'' in \emph{{ICML}}, ser. Proceedings of Machine Learning Research,
  vol. 139.\hskip 1em plus 0.5em minus 0.4em\relax {PMLR}, 2021, pp. 532--542.

\bibitem{Low-Rank-Adaptive-Label-Smoothing}
A.~Ghoshal, X.~Chen, S.~Gupta, L.~Zettlemoyer, and Y.~Mehdad, ``Learning better
  structured representations using low-rank adaptive label smoothing,'' in
  \emph{{ICLR}}.\hskip 1em plus 0.5em minus 0.4em\relax OpenReview.net, 2021.

\bibitem{Randomized-Smoothing}
E.~Rosenfeld, E.~Winston, P.~Ravikumar, and J.~Z. Kolter, ``Certified
  robustness to label-flipping attacks via randomized smoothing,'' in
  \emph{{ICML}}, ser. Proceedings of Machine Learning Research, vol. 119.\hskip
  1em plus 0.5em minus 0.4em\relax {PMLR}, 2020, pp. 8230--8241.

\bibitem{Co-Saliency-Detection-Label-Smoothing}
K.~Zhang, T.~Li, B.~Liu, and Q.~Liu, ``Co-saliency detection via mask-guided
  fully convolutional networks with multi-scale label smoothing,'' in
  \emph{{CVPR}}.\hskip 1em plus 0.5em minus 0.4em\relax Computer Vision
  Foundation / {IEEE}, 2019, pp. 3095--3104.

\bibitem{LS-for-Recommender-Systems}
H.~Wang, F.~Zhang, M.~Zhang, J.~Leskovec, M.~Zhao, W.~Li, and Z.~Wang,
  ``Knowledge-aware graph neural networks with label smoothness regularization
  for recommender systems,'' in \emph{{KDD}}.\hskip 1em plus 0.5em minus
  0.4em\relax {ACM}, 2019, pp. 968--977.

\bibitem{Circular-Smooth-Label}
X.~Yang and J.~Yan, ``Arbitrary-oriented object detection with circular smooth
  label,'' in \emph{{ECCV} {(8)}}, ser. Lecture Notes in Computer Science, vol.
  12353.\hskip 1em plus 0.5em minus 0.4em\relax Springer, 2020, pp. 677--694.

\bibitem{LS-Mitigate-Label-Noise}
M.~Lukasik, S.~Bhojanapalli, A.~K. Menon, and S.~Kumar, ``Does label smoothing
  mitigate label noise?'' in \emph{{ICML}}, ser. Proceedings of Machine
  Learning Research, vol. 119.\hskip 1em plus 0.5em minus 0.4em\relax {PMLR},
  2020, pp. 6448--6458.

\bibitem{LS-Incompatible-Knowledge-Distillation}
Z.~Shen, Z.~Liu, D.~Xu, Z.~Chen, K.~Cheng, and M.~Savvides, ``Is label
  smoothing truly incompatible with knowledge distillation: An empirical
  study,'' in \emph{{ICLR}}.\hskip 1em plus 0.5em minus 0.4em\relax
  OpenReview.net, 2021.

\bibitem{Batch-Normalization}
S.~Ioffe and C.~Szegedy, ``Batch normalization: Accelerating deep network
  training by reducing internal covariate shift,'' in \emph{{ICML}}, ser.
  {JMLR} Workshop and Conference Proceedings, vol.~37.\hskip 1em plus 0.5em
  minus 0.4em\relax JMLR.org, 2015, pp. 448--456.

\bibitem{Improve-Batch-Normalization}
C.~Summers and M.~J. Dinneen, ``Four things everyone should know to improve
  batch normalization,'' in \emph{{ICLR}}.\hskip 1em plus 0.5em minus
  0.4em\relax OpenReview.net, 2020.

\bibitem{Power-Norm}
S.~Shen, Z.~Yao, A.~Gholami, M.~W. Mahoney, and K.~Keutzer, ``Powernorm:
  Rethinking batch normalization in transformers,'' in \emph{{ICML}}, ser.
  Proceedings of Machine Learning Research, vol. 119.\hskip 1em plus 0.5em
  minus 0.4em\relax {PMLR}, 2020, pp. 8741--8751.

\bibitem{Task-Norm}
J.~Bronskill, J.~Gordon, J.~Requeima, S.~Nowozin, and R.~E. Turner, ``Tasknorm:
  Rethinking batch normalization for meta-learning,'' in \emph{{ICML}}, ser.
  Proceedings of Machine Learning Research, vol. 119.\hskip 1em plus 0.5em
  minus 0.4em\relax {PMLR}, 2020, pp. 1153--1164.

\bibitem{BN-Avoids-Ranks-Collapse}
H.~Daneshmand, J.~M. Kohler, F.~R. Bach, T.~Hofmann, and A.~Lucchi, ``Batch
  normalization provably avoids ranks collapse for randomly initialised deep
  networks,'' in \emph{NeurIPS}, 2020.

\bibitem{Disharmony-Dropout-BN}
X.~Li, S.~Chen, X.~Hu, and J.~Yang, ``Understanding the disharmony between
  dropout and batch normalization by variance shift,'' in \emph{{CVPR}}.\hskip
  1em plus 0.5em minus 0.4em\relax Computer Vision Foundation / {IEEE}, 2019,
  pp. 2682--2690.

\end{thebibliography}

%

\clearpage

\appendices

\section{Robustness Evaluation of TRADES w/ or w/o CTR}
This section provides the robustness evaluation of TRADES with and without CTR on more attack methods. As shown in Table~\ref{tab:accuracies-of-TRADES-on-more-attacks}, in comparison with TRADES without CTR ($\beta$=1), TRADES with CTR ($\beta\geq$1 and bold in Table~\ref{tab:accuracies-of-TRADES-on-more-attacks}) can improve the robustness of the model against more white-box gradient attacks (i.e., MIFGSM, APGD and PGD-SCE), white-box optimization attack (i.e., CW$_\infty$) and black-box attack (i.e., Square). Particularly, TRADES with CTR can significantly improve the robustness of the model against white-box gradient attacks.

\section{Robustness Evaluation of Fast-AT, Free-AT and Madry-AT w/ or w/o CTR}
This section provides the robustness evaluation of Fast-AT, Free-AT and Madry-AT with and without CTR on more attack methods. In comparison with Fast-AT, Free-AT and Madry-AT without CTR, these methods with CTR (bold in Table~\ref{tab:accuracies-of-Fast-Free-Madry-on-more-attacks}) can improve the robustness of the model against more white-box gradient attacks (i.e., MIFGSM, APGD and PGD-SCE) while keeping the comparable white-box optimization attack accuracy and comparable black-box attack accuracy.

\begin{table*}[t]
\begin{center}
\caption{The accuracies (\%) of TRADES with or without CTR on more attacks. Time denotes build time (mins). Bold indicates higher performance. MI denotes MIFGSM.}
\label{tab:accuracies-of-TRADES-on-more-attacks}
\begin{tabular}{ccccccccccccc}
\hline
                               &      & \multicolumn{1}{c|}{}     & \multicolumn{5}{c|}{Multistep}                                                                  & \multicolumn{5}{c}{Cyclic}                                                        \\ \hline
Methods                        & $\beta$ & \multicolumn{1}{c|}{$\gamma$} & MI             & APGD           & PGD-SCE        & CW$_\infty$             & \multicolumn{1}{c|}{Square} & MI             & APGD           & PGD-SCE        & CW$_\infty$             & Square        \\ \hline
TRADES w.o. CTR                & 1    & 0                         & 48.1           & 43.6           & 45.8           & 45.1           & 50.0                        & 48.75          & 44.15          & 46.55          & 46.1           & 50.6          \\ \cline{1-1}
\multirow{9}{*}{TRADES w. CTR} & 1    & 1                         & 49.15          & 44.45          & 46.55          & 45.8           & 52.0                        & 49.45          & 45.1           & 47.35          & 46.5           & 50.6          \\
                               & 1    & 2                         & 49.3           & 45.0           & 46.65          & 45.5           & 50.7                        & 47.85          & 43.6           & 45.8           & 44.8           & 48.9          \\
                               & 1    & 3                         & 49.55          & 44.35          & 46.45          & 44.8           & 49.7                        & 48.4           & 43.7           & 45.6           & 44.65          & 49.8          \\
                               & 1    & 4                         & 44.0           & 38.9           & 40.6           & 39.25          & 45.5                        & 48.45          & 44.15          & 45.0           & 43.15          & 49.0          \\
                               & 1    & 5                         & 40.25          & 34.85          & 35.35          & 32.9           & 40.9                        & 38.4           & 31.5           & 34.9           & 32.7           & 38.7          \\
                               & 2    & 3                         & \textbf{51.2}  & \textbf{46.4}  & \textbf{48.2}  & \textbf{46.45} & \textbf{51.1}               & \textbf{51.15} & \textbf{47.2}  & \textbf{48.5}  & \textbf{47.05} & \textbf{51.0} \\
                               & 3    & 3                         & \textbf{52.85} & \textbf{48.6}  & \textbf{49.5}  & \textbf{47.75} & \textbf{51.5}               & \textbf{53.1}  & \textbf{48.6}  & \textbf{49.9}  & \textbf{48.1}  & \textbf{52.4} \\
                               & 4    & 3                         & \textbf{54.15} & \textbf{49.65} & \textbf{50.7}  & \textbf{48.55} & \textbf{52.6}               & \textbf{53.55} & \textbf{50.5}  & \textbf{51.1}  & \textbf{49.9}  & \textbf{52.9} \\
                               & 5    & 3                         & \textbf{53.1}  & \textbf{49.7}  & \textbf{50.15} & \textbf{48.95} & \textbf{51.7}               & \textbf{54.3}  & \textbf{51.15} & \textbf{51.65} & \textbf{50.5}  & \textbf{52.7} \\ \cline{1-1}
TRADES w.o. CTR                & 6    & 0                         & 55.05          & 51.75          & 52.2           & 50.65          & 54.5                        & 53.55          & 50.6           & 51.45          & 50.35          & 52.8          \\ \cline{1-1}
\multirow{7}{*}{TRADES w. CTR} & 6    & 1                         & \textbf{55.25} & \textbf{51.9}  & \textbf{52.4}  & \textbf{50.65} & \textbf{54.0}               & 53.1           & 50.2           & 50.85          & 49.8           & 53.5          \\
                               & 6    & 2                         & 54.85          & 51.05          & 51.4           & 49.8           & 53.4                        & 54.3           & 50.8           & 51.25          & 49.4           & 53.9          \\
                               & 6    & 3                         & 53.85          & 50.45          & 50.5           & 48.85          & 52.4                        & \textbf{53.55} & \textbf{50.2}  & \textbf{50.95} & \textbf{49.15} & \textbf{52.0} \\
                               & 6    & 4                         & 53.8           & 49.8           & 50.0           & 47.95          & 51.8                        & 53.65          & 50.45          & 50.4           & 48.95          & 52.8          \\
                               & 6    & 5                         & 52.85          & 49.25          & 48.95          & 47.25          & 51.2                        & 53.6           & 49.3           & 49.1           & 47.9           & 50.8          \\
                               & 7    & 1                         & \textbf{53.95} & \textbf{51.3}  & \textbf{51.55} & \textbf{49.8}  & \textbf{53.8}               & \textbf{53.95} & \textbf{51.9}  & \textbf{52.0}  & \textbf{50.9}  & \textbf{53.1} \\
                               & 8    & 1                         & 55.4           & 51.9           & 52.3           & 50.8           & 53.7                        & 53.65          & 51.35          & 51.45          & 50.3           & 53.0          \\ \hline
\end{tabular}
\end{center}
\end{table*}
\begin{table*}[t]
\begin{center}
\caption{The accuracies (\%) of Fast-AT, Free-AT and Madry-AT with or without CTR on more attacks. Time denotes build time (mins). Bold indicates higher performance. Note that due to the overfit of the model, Fast-AT without CTR runs 40 epochs on Multistep. MI denotes MIFGSM.}
\label{tab:accuracies-of-Fast-Free-Madry-on-more-attacks}
\begin{tabular}{ccccccccccccc}
\hline
                                  &      & \multicolumn{1}{c|}{}       & \multicolumn{5}{c|}{Multistep}                                                                  & \multicolumn{5}{c}{Cyclic}                                                        \\ \hline
Methods                           & $\gamma$ & \multicolumn{1}{c|}{Epochs} & MI             & APGD           & PGD-SCE        & CW$_\infty$             & \multicolumn{1}{c|}{Square} & MI             & APGD           & PGD-SCE        & CW$_\infty$             & Square        \\ \hline
Fast-AT w.o. CTR                  & 0    & 40                          & 48.15          & 44.25          & 45.6           & 45.8           & 50.2                        & 47.15          & 41.85          & 44.95          & 45.55          & 51.2          \\ \cline{1-1}
\multirow{7}{*}{Fast-AT w. CTR}   & 1    & 50                          & 49.75          & 44.45          & 47.2           & 46.95          & 51.6                        & 47.35          & 42.5           & 44.6           & 45.8           & 49.9          \\
                                  & 2    & 50                          & \textbf{50.0}  & \textbf{45.65} & \textbf{46.25} & \textbf{46.7}  & \textbf{50.5}               & \textbf{49.05} & \textbf{44.4}  & \textbf{45.85} & \textbf{46}    & \textbf{49.4} \\
                                  & 3    & 50                          & 50.9           & 46.45          & 45.8           & 44.1           & 49.8                        & 50.8           & 45.9           & 45.65          & 44.6           & 49.2          \\
                                  & 4    & 50                          & \textbf{54.2}  & \textbf{51.2}  & \textbf{47.6}  & \textbf{46.2}  & \textbf{50.3}               & \textbf{54.15} & \textbf{50.35} & \textbf{47.7}  & \textbf{46.15} & \textbf{49.6} \\
                                  & 5    & 50                          & 52.8           & 49.9           & 46.45          & 44.65          & 47.3                        & 52.25          & 50.25          & 46.95          & 45.0           & 48.2          \\
                                  & 4    & -                           & 54.0           & 49.55          & 47.8           & 46.5           & 50.2                        & 52.8           & 49.2           & 47.35          & 45.7           & 49.0          \\
                                  & 5    & -                           & 52.9           & 48.75          & 47.2           & 45.4           & 48.6                        & 54.05          & 50.45          & 48.75          & 46.4           & 49.4          \\ \hline
Free-AT w.o. CTR                  & 0    & 50                          & 50.7           & 47.45          & 48.25          & 47.65          & 50.6                        & 51.7           & 48.55          & 49.05          & 48.7           & 53.5          \\ \cline{1-1}
\multirow{8}{*}{Free-AT w. CTR}   & 1    & 50                          & 51.6           & 48.25          & 48.4           & 47.85          & 51.7                        & \textbf{52.55} & \textbf{49.95} & \textbf{49.35} & \textbf{49.35} & \textbf{52.0} \\
                                  & 2    & 50                          & 52.5           & 49.35          & 47.8           & 47.15          & 51.8                        & 53.5           & 50.85          & 49.6           & 48.6           & 52.4          \\
                                  & 3    & 50                          & 53.55          & 50.8           & 48.3           & 46.55          & 51.1                        & 55.15          & 52.85          & 50.15          & 48.8           & 51.5          \\
                                  & 4    & 50                          & 51.8           & 49.3           & 45.95          & 44.1           & 46.7                        & 55.2           & 53.3           & 49.6           & 47.2           & 51.4          \\
                                  & 1    & -                           & 51.7           & 48.85          & 48.1           & 47.65          & 51.7                        & \textbf{52.25} & \textbf{49.5}  & \textbf{48.95} & \textbf{48.2}  & \textbf{52.7} \\
                                  & 2    & -                           & \textbf{52.7}  & \textbf{50.1}  & \textbf{49.05} & \textbf{48.4}  & \textbf{52.9}               & \textbf{53.65} & \textbf{51.0}  & \textbf{49.15} & \textbf{48.5}  & \textbf{52.9} \\
                                  & 3    & -                           & \textbf{54.75} & \textbf{51.65} & \textbf{49.3}  & \textbf{47.85} & \textbf{51.8}               & 55.4           & 52.95          & 50.95          & 48.45          & 53.2          \\
                                  & 4    & -                           & 52.75          & 49.9           & 47.0           & 44.85          & 49.1                        & 56.65          & 54.35          & 50.55          & 48.7           & 52.0          \\ \hline
Madry-AT w.o. CTR                 & 0    & 50                          & 52.0           & 49.35          & 49.7           & 49.15          & 53.9                        & 51.15          & 47.6           & 48.65          & 48.95          & 52.8          \\ \cline{1-1}
\multirow{10}{*}{Madry-AT w. CTR} & 1    & 50                          & 53.35          & 50.75          & 49.8           & 49.7           & 52.6                        & 52.6           & 49.35          & 49.35          & 48.9           & 51.3          \\
                                  & 2    & 50                          & 54.85          & 52.1           & 50.2           & 49.2           & 53.2                        & 53.55          & 50.75          & 49.15          & 48.25          & 51.8          \\
                                  & 3    & 50                          & 56.8           & 54.45          & 50.9           & 49.05          & 53.1                        & \textbf{56.4}  & \textbf{54.1}  & \textbf{49.7}  & \textbf{48.15} & \textbf{51.5} \\
                                  & 4    & 50                          & 56.0           & 53.85          & 49.5           & 47.4           & 50.4                        & 55.45          & 53.7           & 49.9           & 47.8           & 50.7          \\
                                  & 5    & 50                          & 52.8           & 50.7           & 46.05          & 44.5           & 46.5                        & 52.05          & 50.35          & 45.8           & 44.1           & 46.8          \\
                                  & 1    & -                           & 53.45          & 50.9           & 50.05          & 49.55          & 52.0                        & 52.45          & 48.1           & 48.9           & 48.35          & 52.6          \\
                                  & 2    & -                           & \textbf{54.4}  & \textbf{51.8}  & \textbf{50.05} & \textbf{48.25} & \textbf{52.1}               & 52.4           & 48.25          & 48.7           & 47.7           & 51.3          \\
                                  & 3    & -                           & 57.15          & 54.5           & 51.2           & 49.7           & 54.4                        & 55.05          & 51.5           & 49.65          & 48.2           & 52.2          \\
                                  & 4    & -                           & 56.1           & 53.15          & 49.55          & 48.1           & 51.5                        & 57.55          & 54.15          & 50.25          & 48.65          & 52.1          \\
                                  & 5    & -                           & 55.5           & 52.95          & 48.45          & 46.2           & 50.2                        & 56.05          & 53.35          & 50.25          & 47.6           & 50.3          \\ \hline
\end{tabular}
\end{center}
\end{table*}
%




\end{document}